\documentclass[10pt,journal,compsoc]{IEEEtran}
\ifCLASSOPTIONcompsoc
  \usepackage[nocompress]{cite}
\else
  \usepackage{cite}
\fi
\ifCLASSINFOpdf
\else
\fi
\usepackage{amsmath}
\usepackage{graphicx}
\usepackage{subfigure}
\usepackage{algorithmic}
\usepackage{multirow}
\usepackage{array}
\usepackage{amssymb}
\usepackage{epsfig}
\usepackage{float}
\usepackage{stfloats}
\usepackage{multirow}
\usepackage{subfloat} 
\usepackage{booktabs}
\usepackage{amsthm}
\usepackage{caption} 
\usepackage{amsthm}
\usepackage{mathrsfs}
\usepackage{color}
\usepackage{xcolor}

\usepackage{url}

\newcommand{\std}[1]{{\tiny\(\pm\)#1}}

\usepackage[breaklinks,colorlinks,linkcolor=black,citecolor=black,urlcolor=black]{hyperref}

\newtheorem{theorem}{Theorem}
\newtheorem{definition}{Definition}

\newtheorem{lemma}{Lemma}

\begin{document}
\title{Attention Spiking Neural Networks}
\author{Man~Yao, Guangshe~Zhao, Hengyu~Zhang, Yifan~Hu, Lei~Deng, Yonghong~Tian, Bo~Xu, and Guoqi~Li

\IEEEcompsocitemizethanks{\IEEEcompsocthanksitem M. Yao is with the School of Automation Science and Engineering, Xi'an Jiaotong University, Xi'an, Shaanxi, China, and also with Peng Cheng Laboratory, China.\protect
\IEEEcompsocthanksitem G. Zhao is with the School of Automation Science and Engeneering, Xi'an Jiaotong University, Xi'an, Shaanxi, China.\protect
\IEEEcompsocthanksitem H. Zhang is with Tsinghua Shenzhen International Graduate School, Tsinghua University, Shenzhen, China.\protect
\IEEEcompsocthanksitem Y. Hu, and L. Deng are with Center for Brain-Inspired Computing Research, Department of Precision Instrument, Tsinghua University, Beijing, China.\protect
\IEEEcompsocthanksitem Y. Tian is with Institute for Artificial Intelligence, Peking University, Beijing, China, and also with Peng Cheng Laboratory, China\protect
\IEEEcompsocthanksitem B. Xu and G. Li are with Institute of Automation, Chinese Academy of Sciences, Beijing, China.\protect\\
The corresponding author: Guoqi Li (E-mail:guoqi.li@ia.ac.cn).\protect}

}
\markboth{FOR REVIEW}%
{Shell \MakeLowercase{\textit{et al.}}: Bare Demo of IEEEtran.cls for Computer Society Journals
}


\IEEEtitleabstractindextext{%
\begin{abstract}
Benefiting from the event-driven nature and sparse spiking communication of the brain, spiking neural networks (SNNs) are becoming a promising energy-efficient alternative to traditional artificial neural networks (ANNs). However, the performance gap between SNNs and ANNs has been a significant hindrance to deploying SNNs ubiquitously for a very long period of time. To leverage the full potential of SNNs, we study the effect of attention mechanisms in SNNs, which makes they can focus on important information. We first present our idea of attention in SNNs with a plug-and-play combined module kit, termed the Multi-dimensional Attention (MA) module. Then, a new attention SNN architecture with end-to-end training called "MA-SNN" is proposed, which infers attention weights along the temporal dimension, channel dimension, as well as spatial dimension separately or simultaneously. Based on the existing neuroscience theories, we exploit the attention weights to optimize membrane potentials, which in turn regulate the spiking response in a data-dependent way. At the cost of negligible additional parameters, MA facilitates vanilla SNNs to achieve sparser spiking activity, better performance, and energy efficiency concurrently. Experiments are conducted in event-based DVS128 Gesture/Gait action recognition and ImageNet-1k image classification. On Gesture/Gait, the spike counts are reduced by 84.9\%/81.6\%, the task accuracy and energy efficiency are improved by 5.9\%/4.7\% and 3.4$\times$/3.2$\times$. On ImageNet-1K, we achieve top-1 accuracy of 75.92\% and 77.08\% on single/4-step Res-SNN-104, which are state-of-the-art results in SNNs. Compared with counterpart Res-ANN-104, the performance gap becomes -0.95/+0.21 percent and has 31.8$\times$/7.4$\times$ better energy efficiency. To our best knowledge, this is for the first time, that the SNN community achieves comparable or even better performance compared with its ANN counterpart in the large-scale dataset. To analyze and support the effectiveness of MA-SNN, we theoretically prove that the spiking degradation or the gradient vanishing, which usually holds in general SNNs, can be resolved by introducing the block dynamical isometry theory. We also analyze the efficiency of MA-SNN based on our proposed spiking response visualization method. Our work lights up SNN's potential as a general backbone to support various applications in the field of SNN research, with a great balance between effectiveness and efficiency.
\end{abstract}

}

\maketitle
\IEEEdisplaynontitleabstractindextext
\IEEEpeerreviewmaketitle

\IEEEraisesectionheading{\section{Introduction}\label{sec:introduction}}

\IEEEPARstart{A}{s} the most remarkable neural network, the human brain is incredibly efficient and capable of performing complex pattern recognition tasks, and has always been a source of innovation for artificial neural networks (ANNs) or conventional deep learning models\cite{simonyan_2014_very_deep,he_resnet_2016}. In the recent past, by reasonably emulating the deep hierarchy structure of the visual cortex, deep ANNs obtained powerful representation and brought amazing successes in a myriad of artificial intelligence applications, e.g., compute vision (CV)\cite{krizhevsky2012imagenet}, natural language processing (NLP)\cite{hirschberg2015advances}, medical diagnosis\cite{esteva2017dermatologist}, game playing\cite{silver2016mastering}, etc. Unfortunately, ANNs pay enormous computational costs to achieve such feats. For example, a standard computer performing only recognition among 1,000 different kinds of objects (ImageNet-1K dataset) expends about 250 watts\cite{Nature_2}. By contrast, the human brain can operate with only nearly 20 watts consumption for various impressive achievements (such as simultaneous recognition, reasoning, control, and movement)\cite{laughlin2003communication}. Many real-world platforms, e.g., smartphones, Internet-of-Things devices among others, have resources and battery constraints, which restrict the implementation of deep ANN\cite{sze2017efficient}. To enable intelligence on such platforms, how to exploit the inherent efficient computation paradigm of the biological neural systems to achieve low-power of implementation of neural networks, is of great value.

Spiking neural networks (SNNs) offer an alternative for enabling energy-efficient intelligence, which emulate biological neuronal functionality by adopting binary spiking signals (0-nothing or 1-spiking event) to complete inter-neuron communication\cite{Maass_1997_LIF}. As a kind of neuromorphic computing algorithm, SNNs can be smoothly executed on the sparse neuromorphic chip, by only handling spike-based accumulate (AC) operations, and can avoid computing the zero values of input or activation (i.e., \emph{event-driven})\cite{Nature_2}. Thus, SNNs consume much lower power than ANNs that are dominated by energy-hungry multiply-and-accumulate (MAC) operations on conventional dense computing hardware such as GPUs. With the release of neuromorphic chips like Tianjic\cite{Nature_1}, TrueNorth\cite{2014TrueNorth}, and Loihi\cite{davies2018loihi}, we are not very far from neuromorphic processors becoming a part of everyday life. 

It remains a challenge to directly train large-scale SNNs to achieve comparable performance with counterpart ANNs for real-world pattern recognition tasks. A recent study of directly training builds advancing residual learning to construct large-scale SNNs, and alleviates the performance gap between deep SNNs and ANNs\cite{Hu_2021_MS}. However, the performance gap still exists. On the other hand, computation over multi-time steps\footnote {A time step is the unit of time taken by each input frame to be processed through all layers of the model.} in deep SNNs\cite{zheng_Going_Deeper_SNN_2021,fang_deep_SNN_2021,Hu_2021_MS} not only boosts the training time and simulation hardware costs, but also incurs high inference latency, more overall energy budget, and memory access overhead of fetching membrane potentials. These limitations prohibit the potential effective algorithm design and lessen the energy benefits of SNNs. To break the ice, we urgently need to take novel inspiration from how the brain works and classic deep learning to build more effective and efficient SNNs.

Humans can naturally and effectively find salient regions in complex scenes\cite{itti_1998_human_attention_1}. Motivated by this observation, attention mechanisms have been introduced into deep learning and achieved remarkable success in a wide spectrum of application domains. Current attention in deep learning generally exists in two ways. One is posing a fundamental paradigm shift in the way of executing meta-operator such as using self-attention conduct Transformer\cite{vaswani_2017_attention_is_all_you_need}. The other prefers integrating with the existing classic deep ANNs that work as auxiliary recalibration modules to increase the representation power of the basic model, such as attention convolutional neural network (attention CNNs)\cite{SE_PAMI}. Recently, apart from the classic application in the NLP, the Transformer structure made its grand debut in the CV, and quickly set off an overwhelming wave of pure attention architecture design by its glaring performance in various tasks\cite{liu2021swin}. The success of self-attention facilitates researchers' understanding of different deep learning architectures, including Transformer, CNN, multi-layer perceptron (MLP), etc., and sparks more effective universal network architecture design\cite{2022_transformer_survey}. Moreover, the above two attention practices can be combined together, such as using attention as an independent auxiliary module for the Transformer to focus on informative features or patches\cite{yuan2021tokens}.  

In contrast to the rapid development of attention mechanisms in ANNs, the application of attention in the SNN domain remains to be exploited. Existing few works are totally different from the aforementioned attention practices in traditional deep learning. They focus on using SNN to simulate the attention mechanism\cite{chevallier_2008_Attention_SNN_1,neokleous_2011_Attention_SNN_2} or executing SNN model compression by attention\cite{kundu_2021_spike_thrift}. We do not intend to shift the meta-operator of existing SNNs, e.g., replacing convolution (or fully connected) with self-attention, but try to apply the attention as an auxiliary unit in a simple and lightweight way to easily integrate with existing SNN architectures for improving representation power, like attention CNNs. Challenges in adapting attention to SNNs arise from three aspects. Firstly, we must keep the neuromorphic computing characteristic of SNNs, which is the basis of SNN's energy efficiency. Thus, implementing the attention while retaining SNN's event-driven is the primary consideration. Secondly, SNNs are used to process various applications, such as sequential event streams and static images. We need to diverse attention SNN design to cope with different scenarios. Thirdly, binary spiking activity makes deep SNNs suffer from spike degradation\cite{zheng_Going_Deeper_SNN_2021} and gradient vanishing\cite{Hu_2021_MS}, collectively referred to as the \emph{degradation problem}, i.e., an accuracy drop would occur on both the training and test sets when the network deepens. Attention should not make the case worse. 

In visual neuroscience, attention enhances neuronal communication efficacy by modulating synaptic weights\cite{briggs2013attention} and neuronal spiking activity rate\cite{spitzer1988increased} in the noisy sensory environment. To emulate attention in the brain, we employ attention to facilitate optimizing the membrane potential of spiking neurons, which can be equivalent to synaptic alteration and would not disrupt the event-driven nature of SNNs. Our design philosophy is clear, exploiting attention to regulate membrane potentials, i.e., focusing on important features and suppressing unnecessary ones, which in turn affects the spiking activity. In contrast, attention is applied to refine activations in CNNs\cite{SE_PAMI}. The underlying reason is that neurons in CNNs communicate with each other using activations coded in continuous values rather than brain-like spiking activations. To adapt attention SNNs to a variety of application scenarios, we merge multi-dimensional attention with SNN (MA-SNN), including \emph{temporal}, \emph{channel}, and \emph{spatial} dimensions, to learn 'when', 'what' and 'where' to attend, respectively. These attention dimensions are exploited separately or simultaneously according to specific task metric requirements such as latency, accuracy, and energy cost. Classic convolutional block attention module (CBAM)\cite{CBAM} is adopted as the basic module to construct MA-SNN. Furthermore, attention residual SNNs are designed to process the large-scale ImageNet-1K. We exploit the MS-Res-SNN\cite{Hu_2021_MS} as the backbone because of its higher accuracy and shortcut connection manner. We argue that membrane-shortcut in MS-Res-SNN is identical to our motivation for introducing the attention, which can also be seen as a way to optimize the membrane potentials. 

The advantages of MA-SNN exist in three folds. Firstly, by emulating attention in brain, we propose the MA-SNN. Extensive experimental results on various tasks show that optimization of membrane potential in a data-dependent manner by attention can lead to sparser spiking responses and incurs better performance and energy efficiency concurrently, like the human brain. Secondly, we uncover the attention mechanism in MA-SNN. We answer one key problem: how can both effectiveness and energy efficiency be achieved simultaneously in MA-SNN. To address this issue, a new spiking response visualization method is proposed to observe the effect of attention-optimized membrane potential on spiking response. We show that the effectiveness of MA-SNN mainly stems from the proper focusing, just the same as previous CNN works\cite{zhou_2016_CAM,park_2020_attention_bam,SE_PAMI,SimAM}. At the efficiency aspect, MA adaptively inhibits the membrane potentials of the background noise, then these spiking neurons would not be activated. With this point of view, we could explain why a much lower spiking activity rate can be achieved in attention SNN with great energy efficiency. Thirdly, we prove that the degradation problem, which holds in general deep SNNs, can be resolved when adding attention to MS-Res-SNN (i.e., Att-Res-SNN). Specifically, we prove the gradient norm equality\cite{chen2020comprehensive} can be achieved in our attention residual learning by introducing the block dynamical isometry theory, which means that one could train very deep Att-Res-SNN in the same way as in MS-Res-SNN. To summarize, the main contributions of this work are as follows:

\begin{enumerate}
\item [$\bullet$] 
\textbf{Multi-dimensional Attention SNN:} Inspired by the attention mechanisms in neuroscience, we present our idea of attention SNN and propose the MA-SNN, which merges multi-dimensional attention with SNN and inherits the event-driven nature, including temporal, channel, and spatial dimensions to learn 'when', 'what', and 'where' to attend. The sparse spiking activity, performance, and energy efficiency of MA-SNN are verified on the multiple benchmarks under the multi-scale constraints of output latency. Based on the proposed model, now the SNN community is able to achieve comparable or even better performance compared with its ANN counterpart in the large-scale dataset. 
\item [$\bullet$]
\textbf{Understanding and Visualizing of Attention:} Through the proposed spiking response visualization method, it is shown that the effectiveness of MA-SNN mainly stems from proper focusing, and efficiency comes from the improvement of sparsity by inhibiting the membrane potentials of the background noise. Thus, we explain why both the effectiveness and efficiency of attention can be achieved concurrently in SNNs. 
\item [$\bullet$] 
\textbf{Gradient Norm Equality of Att-Res-SNN:} We prove the gradient norm equality\cite{chen2020comprehensive} can be achieved in Att-Res-SNN based on the block dynamical isometry theory. The degradation problem that holds in general deep SNNs can then be resolved when adding attention to MS-Res-SNN. Thus, we are able to train very deep Att-Res-SNN to enhance the potential of SNNs.
\end{enumerate}   

The rest of the paper is organized as follows. Section~\ref{section:relate_work} reports preliminaries. Section~\ref{section:MA-SNN} introduces our MA-SNN. Section~\ref{section:energy_analysis} gives how to evaluate the energy cost of attention SNNs. Section~\ref{section:Experiment} verifies the effectiveness and efficiency of our methods. Section~\ref{section:ablation_study} conducts ablation studies to comprehend the design of MA-SNN. Section~\ref{section:visualizing_attention} understands and visualizes the effectiveness and efficiency of attention SNNs. Section~\ref{section:conclusion} concludes this work.

\section{Preliminaries}\label{section:relate_work}
\textbf{Training Methods of SNNs.} ANN-to-SNN conversion and directly training an SNN are two main routines to train deep SNNs. The basic idea of ANN-to-SNN is that the activation values in a ReLU-based ANN can be approximated by the average ﬁring rates of an SNN under the rate-coding scheme. There is a trade-off issue of accuracy and latency in ANN-to-SNN methods that need sufficient time steps for rate-coding to alleviate approximation errors\cite{wu_2021_progressive}. Although the converted SNN can obtain the smallest accuracy gap with ANN in some large-scale structures, such as VGG and ResNet\cite{stockl2021optimized,li_2021_ann2SNN}, they need a longer time step or complicated training methods, which increases the SNN’s latency and restricts the practical application. Directly training an SNN is another training mode of SNN, which constitutes a continuous relaxation of the non-smooth spiking to enable backpropagation with a surrogate gradient\cite{Neftci_SG_2019}. Compared with ANN-to-SNN, it has a great advantage in the number of time steps and can also be applied to temporal tasks, e.g., event-based datasets. Direct training algorithms are diverse in the selection of coding schemes such as time-coding\cite{comsa2020temporal} and rate-coding\cite{Fang_2021_ICCV}. Time-coding has limited the network scale, and we use the rate-coding direct training method to obtain large-scale SNNs in this paper.  

\textbf{Event-based Vision.} Dynamic vision sensor (DVS), which encodes the time, location, and polarity of the brightness changes for each pixel into event streams with a \textmu s level temporal resolution, poses a new paradigm shift in visual information acquisition. Compared with conventional cameras, the advantages of DVS include\cite{Gallego_2020_DVS_Survey}: requiring fewer resources since the events are only triggered when the intensity changes; a high temporal resolution which can avoid motion blur; a very high dynamic range which makes the DVS able to acquire information from challenging illumination conditions. These characteristics promote the application of DVS in various scenarios, such as high-speed object tracking\cite{High_Speed_Event_Camera_2}, autonomous driving\cite{Auto_driving_1}, low-latency interaction\cite{amir_Gesture_dataset_2017}, etc. Processing events one by one is limited to performance because a single event has little information. The general method is to group event streams with a certain temporal window as alternative representations, e.g., frame-based\cite{yao_2021_TASNN}, graph-based\cite{wang_Gait_PAMI_2021}, etc. In this paper, we adopt the frame-based representation that transforms event streams into high-rate videos, where each frame has many blank (zero) areas. SNN is suitable to process event frames since it can skip the computation of the zero areas in each input frame\cite{Nature_2}.  

\textbf{Attention in CNNs.} Depth, width, and cardinality are three important factors to get rich representation power in CNN architecture design. Apart from these factors, attention is another different aspect of architecture design that increases representation power by focusing on important information. Its significance has been studied extensively in the previous literature. Hu \emph{et al.} \cite{SE_PAMI} pioneered the attention module in CNNs, which first proposed the concept of channel attention and presented SENet for this purpose. Motivated by different channels that usually represent different objects, SENet models the relationship between channels to adaptively refine the weight of each channel, i.e., determining what to pay attention to. Since convolution operations extract informative features by blending cross-channel and spatial information together, Woo \emph{et al.} \cite{CBAM} proposed CBAM that sequentially applies both channel and spatial attention modules to determine what and where to pay attention to concurrently. There are many optimizations from various aspects for the modeling of the channel and spatial attention, such as effectiveness\cite{gao2019global}, space complexity\cite{SimAM}, computational complexity\cite{Wang_2020_ECA}, etc. Please refer to \cite{guo_2022_attention_CV_survey} for a comprehensive review of attention in CV.

\begin{figure*}[ht]
\centering
\subfigure[Conv-based LIF-SNN layer]{\includegraphics[scale=0.42]{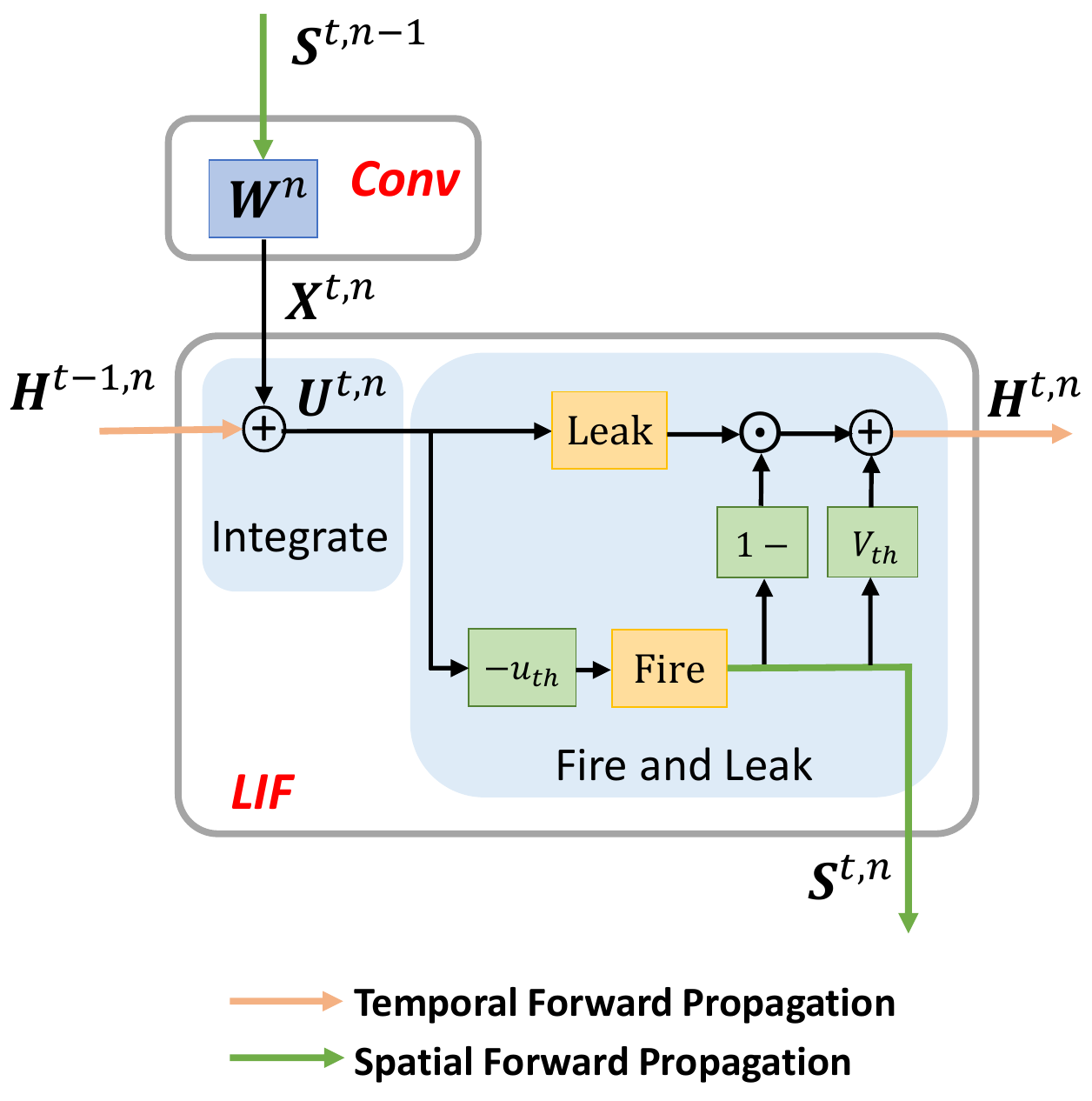}}
\quad \quad \quad 
\subfigure[Multi-dimensional attention SNN]{\includegraphics[scale=0.4]{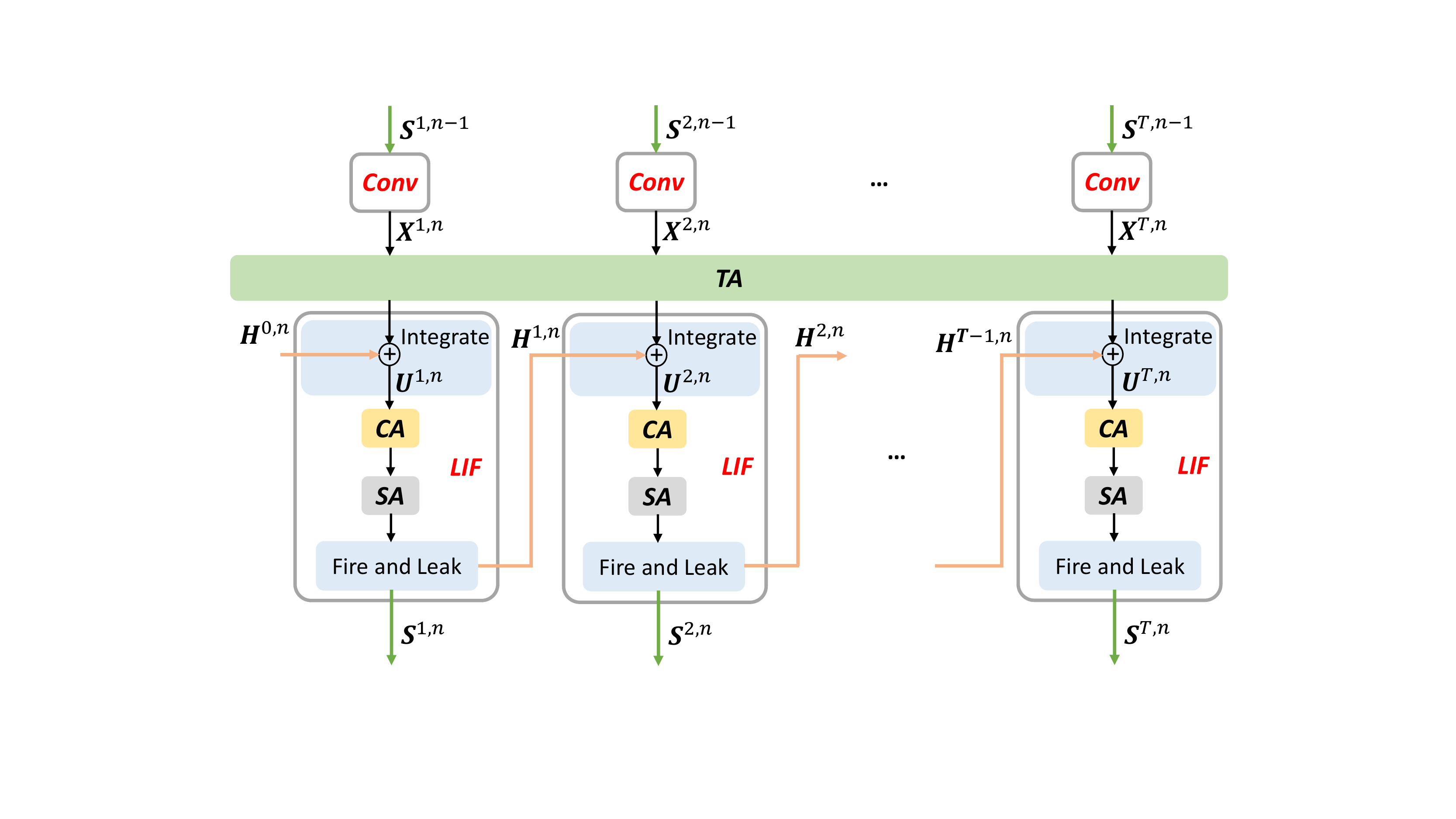}}
\caption{The Conv-based SNN layer and the overview of MA-SNN.}
\label{Fig:LIF_Layer}
\end{figure*}

\section{Multi-dimensional Attention Spiking Neural Networks}\label{section:MA-SNN}
In this section, we introduce the network input preprocessing and the Conv-based SNN in Section~ \ref{subsec:network_input} and Section~\ref{subsec:SNN}, respectively. Then we design the MA module that learns temporal (when), channel (what), and spatial-wise (where) attention separately in Section~\ref{subsec:MA_module}. Finally, we give our design of attention residual learning for SNN in Section~\ref{subsec:attention_residual_SNN}.   

\subsection{Network Input}\label{subsec:network_input}
\textbf{Event-based streams.} We adopt the frame-based representation as the preprocessing method, which transforms event streams into high-rate frame sequences where each frame has many zero areas. Event stream comprises four dimensions: two spatial coordinates $(x, y)$, the timestamp, and the polarity of each single event. The polarity $p$ indicates an increase (ON) or decrease (OFF) of brightness, where ON/OFF can be represented via +1/-1 values. Assume the initial temporal resolution of event stream is $dt^\prime$ (\textmu s level) and the spatial resolution is $h_{0}\times w_{0}$, the spike pattern tensor $\boldsymbol{S}_{t^\prime}\in\boldsymbol{R}^{2 \times h_{0}\times w_{0}}$ is equal to events set $E_{t^\prime}=\left\{e_i|e_i=\left[x_i,y_i,t^\prime,p_i\right]\right\}$ at timestamp $t^\prime$. We can set a new millisecond-level temporal resolution $dt={dt^\prime}\times \alpha$, and consecutive $\alpha$ spike patterns can be grouped as a set
\begin{equation}
    E_{t}=\left\{\boldsymbol{S}_{t^\prime}\right\},
    \label{eq:spike_pattern_set}
\end{equation}
where ${t^\prime}\in\left[\alpha \times t,\alpha\times\left(t+1\right)-1\right]$. Then, the frame for input layer at $t$ time $\boldsymbol{S}^{t, 0}\in\boldsymbol{R}^{2 \times h_{0}\times w_{0}}$ based on $dt$ can be got by
\begin{equation}
    \boldsymbol{S}^{t, 0} = q(E_{t}) = \sum\limits_{t^\prime=\alpha \times t}^{\alpha\times\left(t+1\right)-1} \boldsymbol{S}_{t^\prime},
    \label{eq:event_frame_slice}
\end{equation}
where $t\in\left\{1,2,\cdots,T\right\}$ is the \emph{time step}, and $q(\cdot)$ is element-wise addition function. In this way, the event stream can be transformed into a sequence of real-valued frames with a new frame rate, e.g., $dt = 1 ms$ corresponds to $10^3$ frames per second. Fig.~\ref{Fig:Attention_Feature_Map} shows an example of event frames. 

\textbf{Static images.} For the analog-valued signal of pixel intensity in images, adding an encoding layer to generate spike signals globally is a generic method of SNN\cite{Deng_2020_rethink_ann_SNN}. Since the SNN is a kind of spatio-temporal model, the image is copied and used as input frame at each time step when $T > 1$, i.e., $\boldsymbol{S}^{1, 0} = \boldsymbol{S}^{2, 0} = \cdots = \boldsymbol{S}^{T, 0}$. 

\subsection{Spiking Neural Networks}\label{subsec:SNN}

\textbf{Spiking Neuron.} Spiking neurons, the basic compute units of SNN, communicate through spikes coded in binary activations. The leaky integrate-and-fire (LIF) model is one of the most commonly used spiking neuron models, since it is a trade-off between the complex dynamic characteristics of biological neurons and the simplified mathematical form. It is suitable for simulating large-scale SNN and can be described by a differential function\cite{Maass_1997_LIF}
\begin{equation}
    \tau\frac{du\left(t\right)}{dt}=-u\left(t\right)+I\left(t\right),  \label{eq:continuous LIF model}
\end{equation}
where $\tau$ is a time constant, and $u\left(t\right)$ and $I\left(t\right)$ are the membrane potential of the postsynaptic neuron and the input collected from presynaptic neurons, respectively. 

\textbf{Conv-based LIF-SNN.} Solving this differential equation, a simple iterative representation of LIF-SNN layer \cite{Wu_STBP_2018,Neftci_SG_2019} for easy inference and training is governed by
\begin{equation}
    \left\{\begin{array}{l}
    \boldsymbol{U}^{t, n}=\boldsymbol{H}^{t-1, n}+\boldsymbol{X}^{t, n} \\
    \boldsymbol{S}^{t, n}=\operatorname{Hea}\left(\boldsymbol{U}^{t, n}-u_{t h}\right) \\
    \boldsymbol{H}^{t, n}=V_{reset}\boldsymbol{S}^{t, n} + \left(\beta \boldsymbol{U}^{t, n}\right) \odot \left(\mathbf{1}-\boldsymbol{S}^{t, n}\right), \\
    \end{array}\right. \label{eq:SNN_layer}
\end{equation}
where $t$ and $n$ denote the time step and layer, $\boldsymbol{U}^{t, n}$ means the membrane potential which is produced by coupling the spatial feature $\boldsymbol{X}^{t, n}$ and the temporal input $\boldsymbol{H}^{t-1, n}$, $u_{th}$ is the threshold to determine whether the output spiking tensor $\boldsymbol{S}^{t, n}$ should be given or stay as zero, $\operatorname{Hea}(\cdot)$ is a Heaviside step function that satisfies $\operatorname{Hea}\left(x\right)=1$ when $x\geq0$, otherwise $\operatorname{Hea}\left(x\right)=0$, $V_{reset}$ denotes the reset potential which is set after activating the output spiking, and $\beta = e^{-\frac{d t}{\tau}} < 1$ reflects the decay factor, and $\odot$ denotes the element-wise multiplication. 

In Eq. \ref{eq:SNN_layer}, spatial feature $\boldsymbol{X}^{t, n}$ can be extracted from the original input $\boldsymbol{S}^{t, n-1}$ through a convolution operation:  
\begin{equation}
    \boldsymbol{X}^{t, n} = \operatorname{AvgPool}\left(\operatorname{BN}\left(\operatorname{Conv}\left(\boldsymbol{W}^{n}, \boldsymbol{S}^{t, n-1}\right)\right)\right), \\
    \label{eq:Conv-based}
\end{equation}
where $\operatorname{AvgPool}(\cdot)$, $\operatorname{BN}(\cdot)$ and $\operatorname{Conv}(\cdot)$ mean the average pooling, batch normalization\cite{ioffe_batchNorm_2015} and convolution operation respectively, $\boldsymbol{W}^{n}$ is the weight matrix, $\boldsymbol{S}^{t, n-1}(n \neq 1)$ is a spike tensor that only contains 0 and 1, and $\boldsymbol{X}^{t, n} \in\mathbb{R}^{c_{n} \times h_{n} \times w_{n}}$. To simplify the notation, bias terms are omitted. BN is a default operation following the Conv, we also omit it in the rest of this paper.

\textbf{Fire and Leak Mechanism.} As shown in Fig. \ref{Fig:LIF_Layer}(a), the CNN-based LIF-SNN layer consists of two parts, the Conv and LIF. The Conv module extracts spatial features from original input $\boldsymbol{S}^{t, n-1}$ firstly. Then the LIF module integrates the spatial feature $\boldsymbol{X}^{t, n}$ and the temporal input $\boldsymbol{H}^{t-1, n}$ into membrane potential $\boldsymbol{U}^{t, n}$. Finally, the fire and leak mechanism is exploited to generate spatial spiking tensor for the next layer and the new cell states for the next time step. When the entries in $\boldsymbol{U}^{t, n}$ are greater than the threshold $u_{th}$, the spatial output of spiking sequence $\boldsymbol{S}^{t, n}$ will be activated, and the entries in $\boldsymbol{U}^{t, n}$ will be reset to $V_{reset}$, then the temporal output $\boldsymbol{H}^{t, n}$ should be decided by the $\boldsymbol{X}^{t, n}$ since $\mathbf{1}-\boldsymbol{S}^{t, n}$ must be 0. Otherwise, the decay of the $\boldsymbol{U}^{t, n}$ will be used to transmit the $\boldsymbol{H}^{t, n}$, since the $\boldsymbol{S}^{t, n}$ is 0, which means there is no activated spiking output. Note that, after the convolution operation, all tensors in the LIF module have the same dimensions, i.e., $\boldsymbol{X}^{t, n}, \boldsymbol{H}^{t-1, n},\boldsymbol{U}^{t, n},\boldsymbol{S}^{t, n},\boldsymbol{H}^{t, n}\in\boldsymbol{R}^{c_{n} \times h_{n} \times w_{n}}$. 

\begin{figure*}[ht]
\centering
\subfigure[Temporal-wise attention]{\includegraphics[scale=0.42]{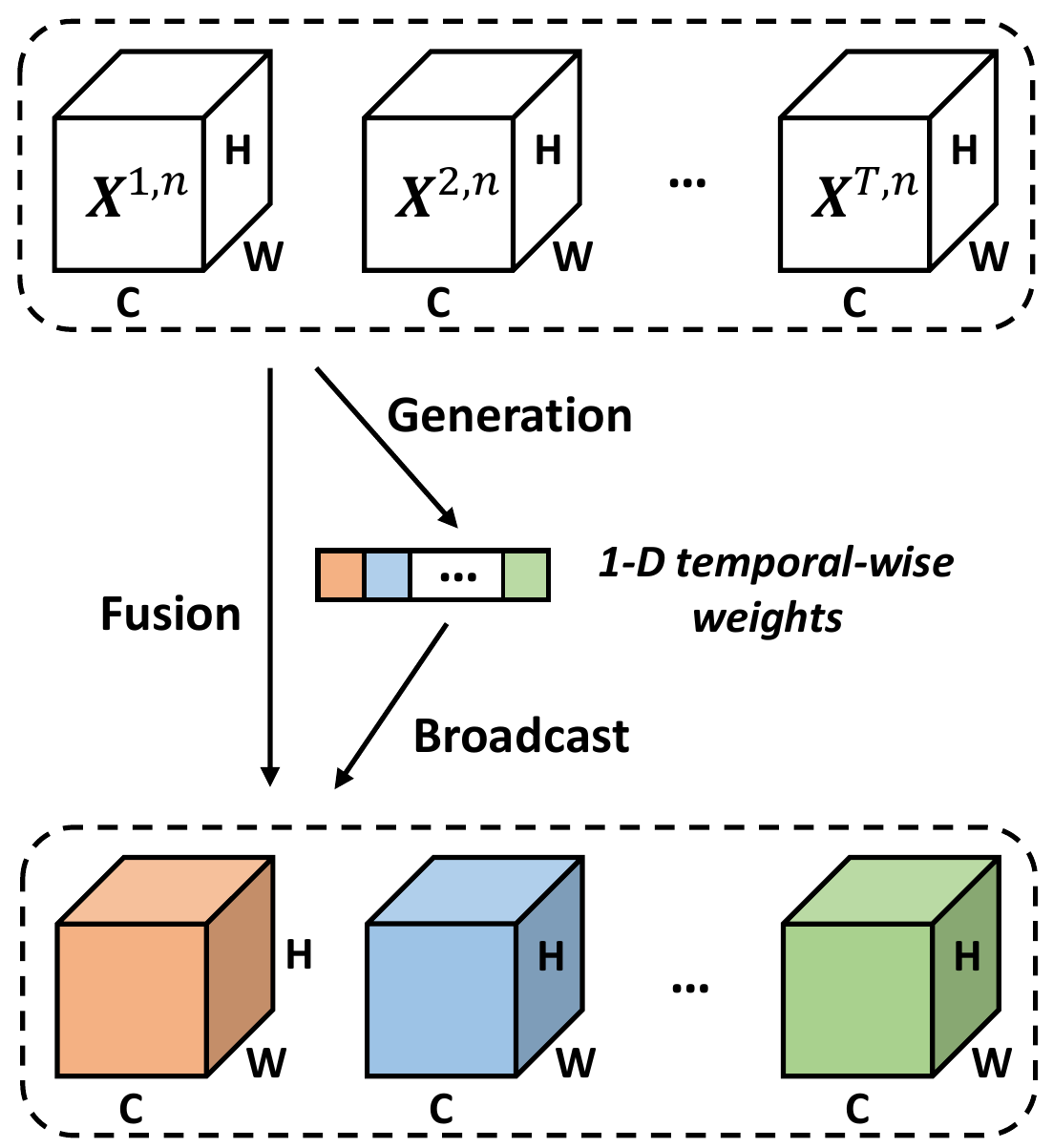}} \quad \quad \quad \quad
\subfigure[Channel-wise attention]{\includegraphics[scale=0.42]{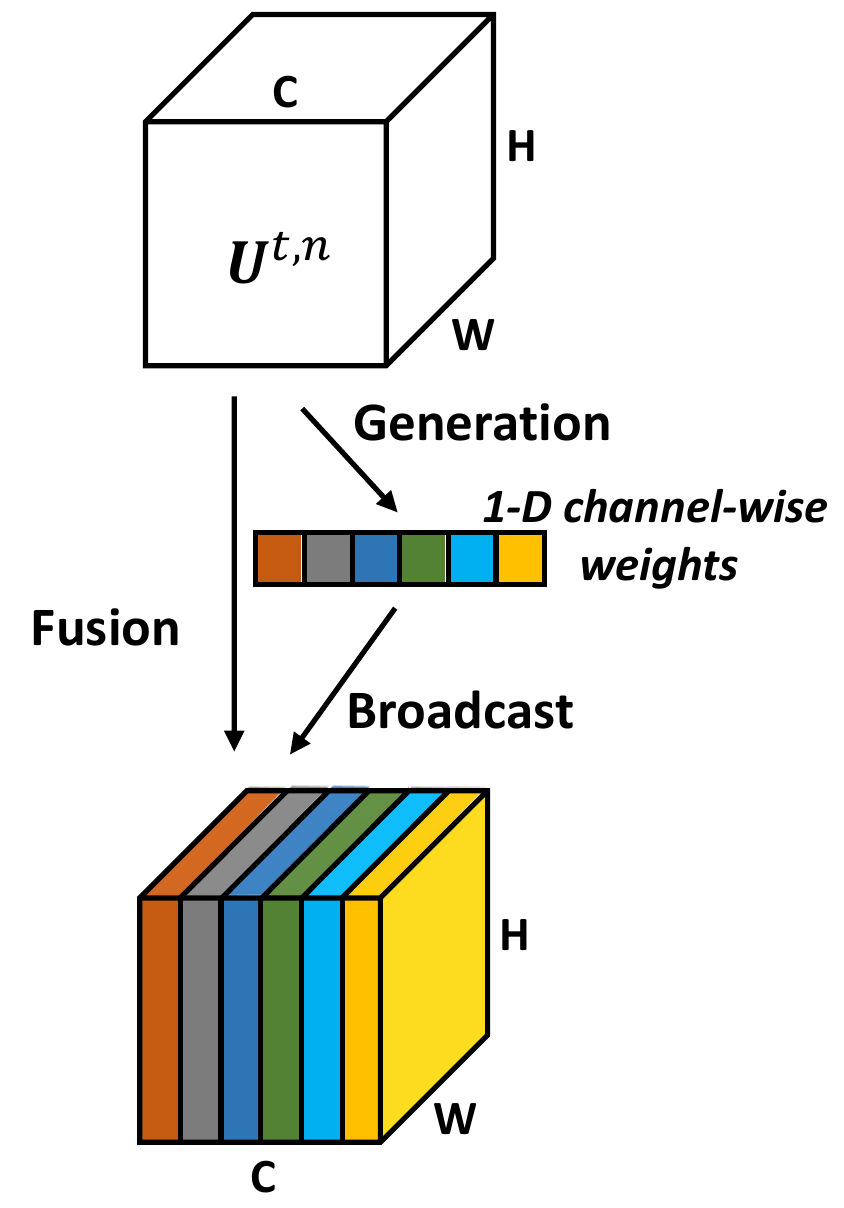}}\quad \quad \quad \quad
\subfigure[Spatial-wise attention]{\includegraphics[scale=0.42]{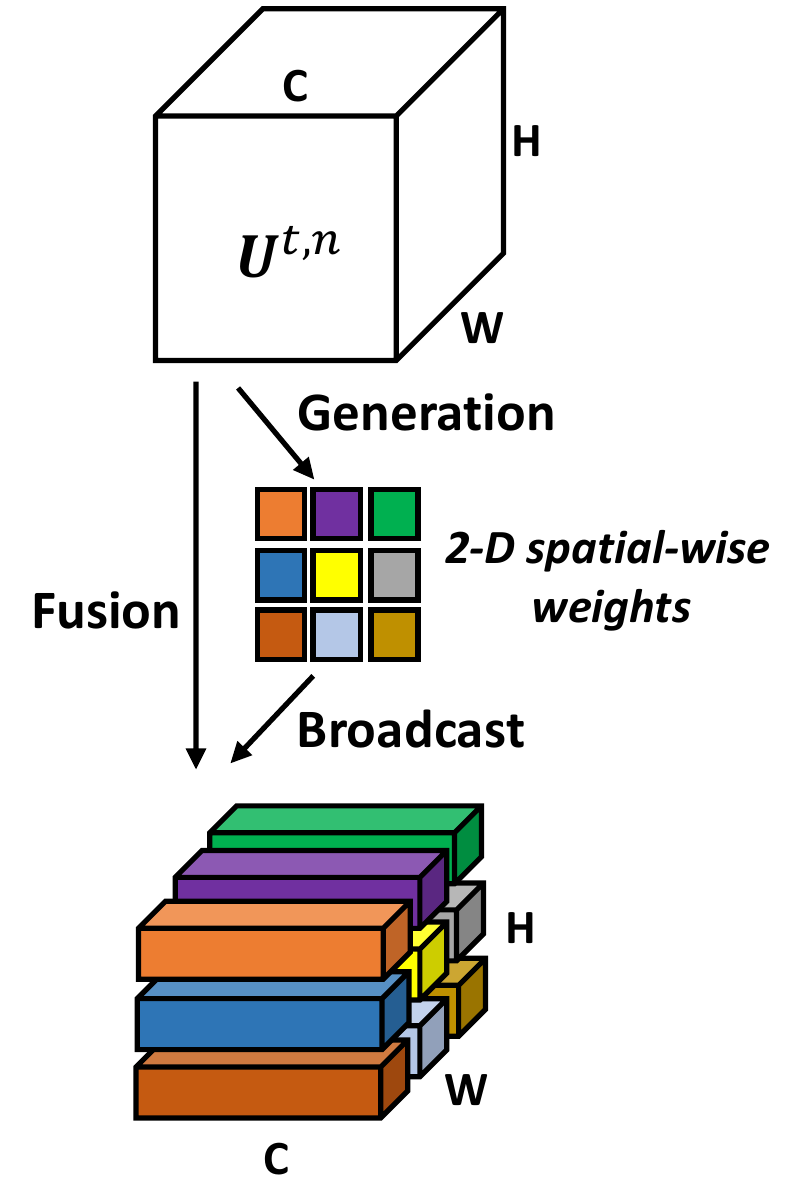}}
\caption{Illustration of different attention dimensions.}
\label{Fig:attention_mechansim}
\end{figure*}

\subsection{Multi-dimensional Attention for SNNs}\label{subsec:MA_module}

\textbf{Attention in Neuroscience.} Selective attention is a powerful brain mechanism that enables enhanced processing of relevant information while preventing interference from distracting noise\cite{moran1985selective}. Many studies in humans and animals have investigated the effect of visual attention in single-neuron, and given two solid observations\cite{briggs2013attention,maunsell2015neuronal}. The first conclusion is about how attention modulates neuronal communication with each other\cite{briggs2013attention}. By altering dendritic spines (building blocks of synapse, i.e., synaptic weights\cite{engert1999dendritic}) in a dynamic and highly selective manner, neurons enhance the detection of salient information in the noisy sensory environment. Another consensus is that attention is associated with neuronal spiking activity rate, not only in local single-neuron\cite{briggs2013attention} but also in more global visual cortex areas\cite{mitchell2009spatial}. Inputs that carry salient sensory information enhance the spiking activity rates of neurons, while potentially redundant input information does the opposite effection.

\textbf{Overview of Attention.} Although the idea of attention in CNNs is the same as neuroscience, there are fundamental differences. Obviously, continuous activations of CNNs do not conform to the spiking activation properties in biological neurons, thus losing the potential energy efficiency earnings caused by attention. To mimic the attention that modulates the spiking activity of neurons in brain, we optimize the membrane potential of spiking neurons by attention in a data-dependent manner, and the spiking response of SNN is consequently regulated. Generally, we can formulate attention processes as:
\begin{equation}
    x_{Att} = f(g(x), x), \\
    \label{eq:generic_attention_eq}
\end{equation}
where $x_{Att}$ is output with an attention mechanism, $g(x)$ is the function of generating attention weights which corresponds to the process of attending to the discriminative moments or regions. $f(g(x), x)$ means processing input $x$ based on the attention weights $g(x)$.

As shown in Fig.~\ref{Fig:LIF_Layer}(b), we design a multi-dimensional attention module that learns temporal (when), channel (what), and spatial (where) attention separately. Attention modules for each dimension can be exploited in an independently or jointly way. For MA of this paper, we adopt
\begin{equation}
    x_{Att} = g(x) \cdot x, \\
    \label{eq:generic_attention_eq_2}
\end{equation}
and input $x$ is usually an intermediate feature map.

\textbf{Temporal-wise Attention (TA).} Consider the event-based visual recognition in a real-time interaction scenario, where an event stream is divided into high-rate frames in sequence while a prediction can be retrieved after processing the data with $t_{lat} = dt \times T$ ms. Our previous research \cite{yao_2021_TASNN} observed that the accuracy of the SNN would not become worse even they masked half of the input event frames, and proposed a lightweight temporal-wise attention SNN to handle event streams by discarding irrelevant event frames. However, directly masking the input frames just keeps accuracy unchanged. By contrast, we use an advanced method that refines the interior features of the network in temporal dimension by exploiting the inter-time step relationship of feature blocks, which could obtain performance gain and energy reduction concurrently.

As depicted in Fig.~\ref{Fig:attention_mechansim}(a), to compute the 1-D TA weights, we collect intermediate feature blocks of $n$-th layer at all time steps $\boldsymbol{X}^{n} = \left[\cdots, \boldsymbol{X}^{t,n}, \cdots \right] \in \mathbb{R}^{T \times c_{n} \times h_{n} \times w_{n}}$ as input. TA function $g_{t}(\cdot)$ infers a 1-D TA weight vector $g_{t}(\boldsymbol{X}^{n}) \in \mathbb{R}^{T \times 1 \times 1 \times 1}$ as
\begin{equation}
    \boldsymbol{X}_{TA}^{n} = g_{t}(\boldsymbol{X}^{n})\odot \boldsymbol{X}^{n}, \\
    \label{eq:TA_2}
\end{equation}
where $\boldsymbol{X}_{TA}^{n} = \left[\cdots , \boldsymbol{X}_{TA}^{t,n}, \cdots \right] \in \mathbb{R}^{T \times c_{n} \times h_{n} \times w_{n}}$ is the temporal-wise refined feature blocks. During multiplication, the TA weights are broadcasted (copied) accordingly, along both the channel and spatial dimension.

Inspired by the squeeze-and-excitation operation in previous classic channel-wise attention module designs\cite{SE_PAMI,CBAM}, we design TA function $g_{t}(\cdot)$ in a similar manner. We first aggregate spatial-channel information of a feature block at each time step by using both average-pooling and max-pooling operations, generating two different temporal context descriptors, which denote average-pooled features and max-pooled features respectively. Then, we transform both average-pooled and max-pooled features to a TA weight vector by a shared MLP network, i.e.,
\begin{equation}
\begin{aligned}
g_{t}(\boldsymbol{X}^{n}) = & \sigma\left(\boldsymbol{W}_{t1}^{n}(\operatorname{ReLU}(\boldsymbol{W}_{t0}^{n}(\operatorname{AvgPool}(\boldsymbol{X}^{n})))) \right.\\
& \left. +\boldsymbol{W}_{t1}^{n}(\operatorname{ReLU}(\boldsymbol{W}_{t0}^{n}(\operatorname{MaxPool}(\boldsymbol{X}^{n}))))\right),
\end{aligned}\label{Eq:TA_1}
\end{equation}
where $\operatorname{AvgPool}(\boldsymbol{X}^{n}), \operatorname{MaxPool}(\boldsymbol{X}^{n}) \in \mathbb{R}^{T\times 1 \times 1 \times 1}$ represent the results of average-pooling and max-pooling layer respectively, $\sigma$ means the sigmoid function, $\boldsymbol{W}_{t0}^{n}\in \mathbb{R}^{\frac{T}{r_{t}} \times T}$ and $\boldsymbol{W}_{t1}^{n}\in \mathbb{R}^{T \times \frac{T}{r_{t}}}$ are the weights of linear layers in the shared MLP, $r_{t}$ denotes the temporal dimension reduction factor used to control the computational burden of MLP. 

\textbf{Channel-wise Attention (CA).} CA in CNNs only works on spatial features. Discriminatively, our CA design optimizes spatio-temporal fusion information of SNN, i.e., we adopt the CA function $g_{c}(\cdot)$ to directly refine the membrane potential of spiking neurons (more discusses of CA location design are given in Section~\ref{subsec:attention_location}). It is well known that each channel of feature maps corresponds to a certain visual pattern, and CA focuses on "what" are salient semantic attributes for the given input. Interestingly, we find that another key role of attention is the suppression of minor features, which is usually ignored in attention CNN but crucial for the efficiency of the SNN (details in Section~\ref{section:visualizing_attention}). 


We adopt the classic CBAM\cite{CBAM} to generate the CA vector for SNN at each time step as
\begin{equation}
\begin{aligned}
g_{c}(\boldsymbol{U}^{t,n}) = & \sigma\left(\boldsymbol{W}_{c1}^{n}(\operatorname{ReLU}(\boldsymbol{W}_{c0}^{n}(\operatorname{AvgPool}(\boldsymbol{U}^{t, n})))) \right.\\
& \left. +\boldsymbol{W}_{c1}^{n}(\operatorname{ReLU}(\boldsymbol{W}_{c0}^{n}(\operatorname{MaxPool}(\boldsymbol{U}^{t, n}))))\right),
\end{aligned}\label{eq:CA_1}
\end{equation}
where $g_{c}(\boldsymbol{U}^{t,n})\in\boldsymbol{R}^{c_{n} \times 1 \times 1}$ is the 1-D CA weights, $\operatorname{AvgPool}(\boldsymbol{U}^{t, n}), \operatorname{MaxPool}(\boldsymbol{U}^{t, n}) \in \mathbb{R}^{c_{n}\times 1 \times 1}$, $\boldsymbol{W}_{c0}^{n}\in \mathbb{R}^{\frac{c_{n}}{r_{c}} \times c_{n}}$, $\boldsymbol{W}_{c1}^{n}\in \mathbb{R}^{c_{n} \times \frac{c_{n}}{r_{c}}}$, and $r_{c}$ represents the channel dimension reduction factor. To reduce parameters, the MLP weights, $\boldsymbol{W}_{c0}^{n}$ and $\boldsymbol{W}_{c1}^{n}$, are shared for different time steps. The refined feature $\boldsymbol{U}_{CA}^{t, n}\in\boldsymbol{R}^{c_{n} \times h_{n} \times w_{n}}$ is computed as
\begin{equation}
    \boldsymbol{U}_{CA}^{t, n} = g_{c}(\boldsymbol{U}^{t, n})\odot \boldsymbol{U}^{t, n}. \\
    \label{eq:CA_2}
\end{equation}
Note that, the only difference between TA and CA is the inputs of attention, the former is $\boldsymbol{X}^{n}$, the latter is $\boldsymbol{U}^{t, n}$ (see Fig.~\ref{Fig:attention_mechansim}). Actually, $g_{t}(\cdot)$ and $g_{c}(\cdot)$ can be acted by various attention modules, including classic SE\cite{SE_PAMI}, energy-efficient ECA\cite{Wang_2020_ECA}, and parameter-free SimAM\cite{SimAM}, etc. We conduct ablation studies of these choices in Section~\ref{subsec:choice_of_attention_module}.

\textbf{Spatial-wise Attention (SA).} Different from the above TA and CA, the SA focuses on "where" is an informative part. We adopt the SA part of CBAM\cite{CBAM} as our SA function $g_{s}(\cdot)$, which is described as: \begin{equation}
\begin{aligned}
g_{s}(\boldsymbol{U}^{t, n}) = & \sigma\left(f^{7 \times 7}([\operatorname{AvgPool}(\boldsymbol{U}^{t, n});\operatorname{MaxPool}(\boldsymbol{U}^{t, n})]) )\right.,\\
\end{aligned}\label{eq:SA_1}
\end{equation}
where $\operatorname{AvgPool}(\boldsymbol{U}^{t, n}), \operatorname{MaxPool}(\boldsymbol{U}^{t, n}) \in \mathbb{R}^{1 \times h_{n} \times w_{n}}$, $g_{s}(\boldsymbol{U}^{t, n}) \in \mathbb{R}^{1 \times h_{n} \times w_{n}}$ is the 2-D SA attention weights, and $f^{7 \times 7}$ represents a convolution operation with the filter size of $7 \times 7$ (default hyper-parameter setting in CBAM). Then, the refined feature $\boldsymbol{U}_{SA}^{t, n}$ (see Fig.~\ref{Fig:attention_mechansim}(c)) is computed as
\begin{equation}
    \boldsymbol{U}_{SA}^{t, n} = g_{s}(\boldsymbol{U}^{t, n})\odot \boldsymbol{U}^{t, n}, \\
    \label{eq:SA_2}
\end{equation}
and the convolution operation $f^{7 \times 7}$ is shared for each time step to save parameters.

Finally, when we adopt the above three dimensions of attention concurrently (i.e., TCSA), compared with $\boldsymbol{U}^{t, n}$ of vanilla SNN in Eq.~\ref{eq:SNN_layer}, the \emph{new membrane potential} behaviors of TCSA-SNN layer follow
\begin{equation}
\begin{array}{l}
\boldsymbol{X}_{TA}^{n} = g_{t}(\boldsymbol{X}^{n})\odot \boldsymbol{X}^{n}, \\
\boldsymbol{U}_{CA}^{t, n} = g_{c}(\boldsymbol{H}^{t-1, n}+ \boldsymbol{X}_{TA}^{t,n}) \odot (\boldsymbol{H}^{t-1, n}+ \boldsymbol{X}_{TA}^{t,n}), \\
\boldsymbol{U}^{t, n} = g_{s}(\boldsymbol{U}_{CA}^{t, n})\odot \boldsymbol{U}_{CA}^{t, n}. \\
\end{array}
\end{equation}

\begin{figure}[!t]
\centering
\includegraphics[scale=0.5]{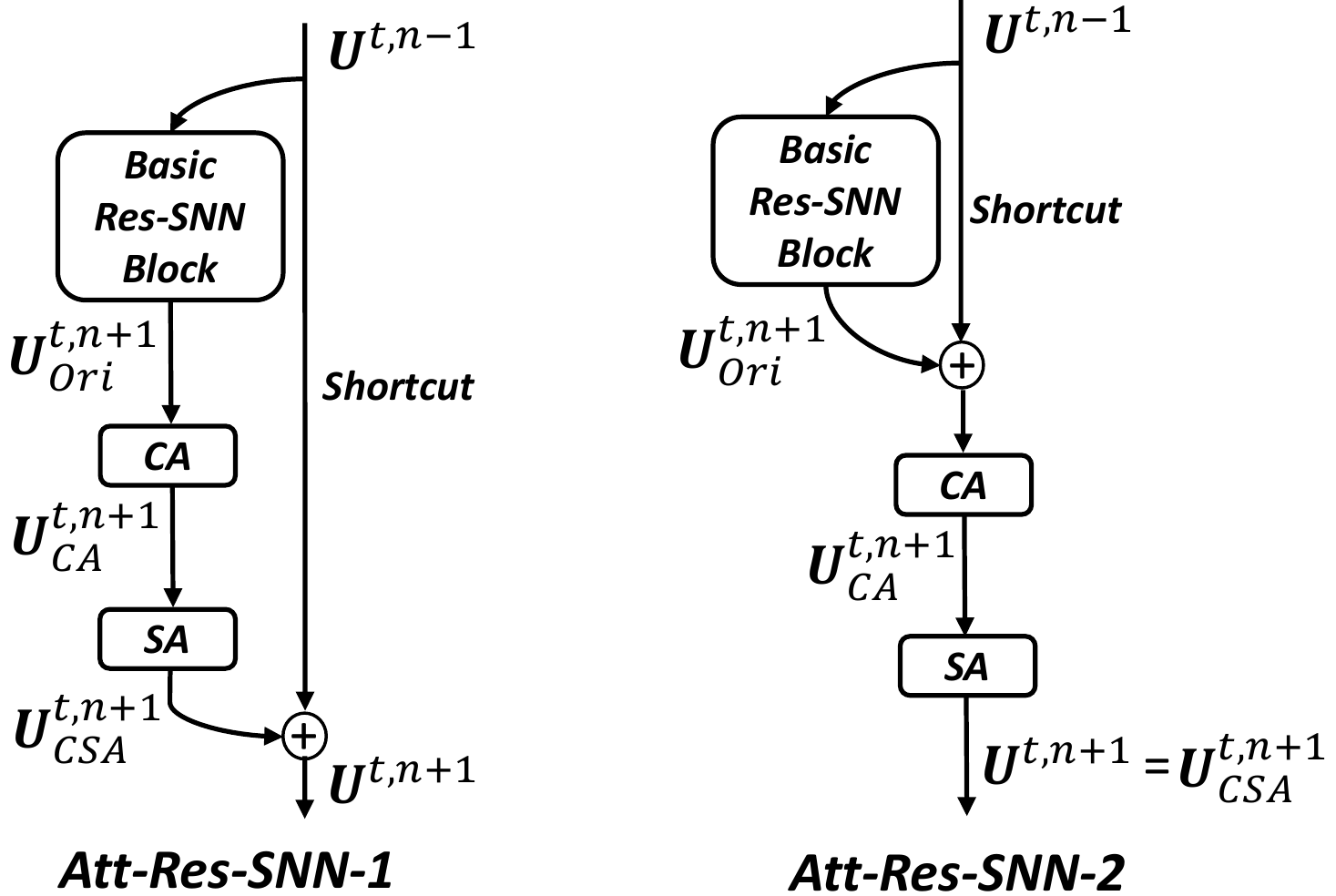}
\caption{Attention residual block contains three parts: basic Res-SNN block (MS-Res-SNN\cite{Hu_2021_MS}, details in Fig.\ref{Fig:attention_res_SNN}), shortcut, and CSA module. We exploit MA on the basic Res-SNN block outputs (i.e., membrane potential of spiking neurons) in each block. We recommend Att-Res-SNN-1 as the scheme for attention residual learning.}
\label{Fig:attention_res_block}
\end{figure}

\subsection{Attention Residual Learning of SNNs}\label{subsec:attention_residual_SNN}
SNNs have been theoretically proved its computational power can be matched with ANNs\cite{Maass_1997_LIF}. However, the limited scale in SNNs restricts the network representation power, which consequently impedes SNN's practice applications and intensifies the performance gap between SNNs and ANNs. Residual learning\cite{he_resnet_2016} becomes a milestone work in deep learning through attaching an identity skip connection throughout the entire network to achieve "very deep" neural networks. Unfortunately, directly copying the classic residual structure to SNNs, there will still be a degradation problem in that the deeper SNNs have higher train loss than the shallower SNNs. There are three mainstream residual structures of SNN, vanilla Res-SNN\cite{zheng_Going_Deeper_SNN_2021}, SEW-Res-SNN\cite{fang_deep_SNN_2021}, and MS-Res-SNN\cite{Hu_2021_MS}. The main difference among these Res-SNN works is the construction of the basic residual blocks (see Fig.\ref{Fig:attention_res_SNN}), and currently there is no uniformly standard residual block scheme in the SNN community. 

MA can be integrated into existing residual SNN architectures without constraints, and we consistently exploit attention to optimize membrane potential of spiking neurons. In this paper, we adopt the MS-Res-SNN\cite{Hu_2021_MS} as the basic residual block. As shown in Fig.~\ref{Fig:attention_res_block}, the channel-spatial attention module can be integrated with existing residual block in two ways. For Att-Res-SNN-1, we have
\begin{equation}
\begin{array}{l}
\boldsymbol{U}_{CA}^{t, n+1} = g_{c}(\boldsymbol{U}_{Ori}^{t, n+1}) \odot \boldsymbol{U}_{Ori}^{t, n+1}, \\
\boldsymbol{U}_{CSA}^{t, n+1} = g_{s}(\boldsymbol{U}_{CA}^{t, n+1}) \odot \boldsymbol{U}_{CA}^{t, n+1}, \\
\boldsymbol{U}^{t, n+1} = \boldsymbol{U}_{CSA}^{t, n+1} + \boldsymbol{U}^{t, n-1}, \\
\end{array}
\end{equation}
and Att-Res-SNN-2 can be described as
\begin{equation}
\begin{array}{l}
\boldsymbol{U}_{CA}^{t, n+1} = g_{c}(\boldsymbol{U}_{Ori}^{t, n+1} + \boldsymbol{U}^{t, n-1}) \odot (\boldsymbol{U}_{Ori}^{t, n+1} + \boldsymbol{U}^{t, n-1}), \\
\boldsymbol{U}_{CSA}^{t, n+1} = g_{s}(\boldsymbol{U}_{CA}^{t, n+1}) \odot \boldsymbol{U}_{CA}^{t, n+1}, \\
\boldsymbol{U}^{t, n+1} = \boldsymbol{U}_{CSA}^{t, n+1}, \\
\end{array}
\end{equation}
where $\boldsymbol{U}_{Ori}^{t, n+1}$ is the original output of the basic Res-SNN block, $\boldsymbol{U}_{CA}^{t, n+1}$ and $\boldsymbol{U}_{CSA}^{t, n+1}$ are the output of the CA and CSA module, respectively, $\boldsymbol{U}^{t, n+1}$ is the final output of Att-Res-SNN block. To keep event-driven nature and avoid degradation deficiency in deep SNNs concurrently, the design of attention location is also important in residual SNNs. We recommend Att-Res-SNN-1 as a scheme for attention residual learning. Analysis of basic Res-SNN block selection and ablation studies of attention residual SNNs are conducted in Section~\ref{subsec:choice_of_attention_residual}. Backpropagation gradient evolvement of Att-Res-SNN-1 and Att-Res-SNN-2 are discussed in Section~\ref{subsec:Gradient_evolvement}.

\section{Analysis of Energy Consumption}\label{section:energy_analysis}
Times of floating-point operations (FLOPs) are used to estimate computational burden in CNNs, where almost all FLOPs are MAC. For an SNN, the measure of energy cost is relatively complicated because FLOPs of the first encoder layer are MAC while all other Conv or FC layers are AC, please see Table~S1 in Supplementary Materials (SM). For MA-SNN, we use MA to regulate membrane potentials, which in turn drops the spiking activity. Thus, the energy increase comes from MAC operations due to attention. The energy decrease comes from the drop of AC operations caused by sparser spiking activity. We here focus on evaluating the energy shift between vanilla and attention SNNs. 

\textbf{Energy Cost of Vanilla SNNs.} In the encoder layer of vanilla SNNs ($n=1$), FLOPs are MAC operations that are the same as CNNs, because the work of this layer is to transform analog inputs into spikes. In addition, all other Conv and FC layers transfer spikes and execute AC operations to accumulate weights of postsynaptic neurons. Thus, the inference energy cost of a vanilla SNN $E_{Base}$ can be quantified as
\begin{equation}
\begin{aligned}
    E_{Base} &= E_{MAC} \cdot FL_{SNNConv}^{1} \\ 
    &+ E_{AC} \cdot (\sum\limits_{n=2}^{N} FL_{SNNConv}^{n} + \sum\limits_{m=1}^{M} FL_{SNNFC}^{m}), 
\end{aligned}\label{eq:base_snn_energy_cost}
\end{equation}
where $N$ and $M$ are the total number of layers of Conv and FC, $E_{MAC}$ and $E_{AC}$ represent the energy cost of MAC and AC operation, $FL_{SNNConv}^{n}$ and $FL_{SNNFC}^{m}$ are the FLOPs of $n$-th Conv and $m$-th FC layer, respectively. Refer to previous SNN works\cite{kundu_2021_HIRE_SNN,yin_2021_NMI,Hu_2021_MS}, we assume the data for various operations are 32-bit floating-point implementation in 45nm technology\cite{horowitz_energy_cost_2014}, in which $E_{MAC} = 4.6pJ$ and $E_{AC} = 0.9pJ$.  

\textbf{Additional Model and Computational Complexity.} Additional parameters and computational burden induced by three dimensions of attention modules are shown in Table~S2 (see SM). The additional parameters are solely from the two FC layers (TA, CA) or one Conv layer (SA), and therefore constitute a small fraction of the total network capacity. The additional computation burden includes two parts, $\Delta_{MAC1}$ and $\Delta_{MAC2}$, where the former comes from generating attention weights and the latter derives from refinement membrane potentials.  

\textbf{Energy Shift of Attention SNNs.} By optimizing the membrane potential, the attention mechanism drops the spiking activity of SNNs in both Conv and FC layers. The computation formulas of AC reduction are shown in Table~S2 of SM. We can estimate the shift of the energy cost versus the additional computational burden $\Delta_{MAC} = \Delta_{MAC1} + \Delta_{MAC2}$ and the decreased AC operations $\Delta_{AC}$ to demonstrate the energy efficiency of attention SNNs. The absolute energy shift between vanilla and attention SNNs can be computed as
\begin{equation}
    \Delta_{E} = E_{MAC} \cdot \Delta_{MAC} - E_{AC} \cdot \Delta_{AC}. \\
\end{equation}
We term the attention SNN energy consumption as $E_{Att}$. With the vanilla SNN as the anchor, the energy efficiency of an attention SNN is defined as 
\begin{equation}
    r_{EE} = \frac{E_{Base}}{E_{Att}} = \frac{E_{Base}}{E_{Base} + \Delta_{E}}. \\
\end{equation}
The higher the $r_{EE}$, the greater the energy efficiency. Generally, we represent the $r_{EE}$ of baseline model as 1$\times$.

\textbf{Network Average Spiking Activity Rate (NASAR).} $r_{EE}$ can estimate energy shift in a fine granularity manner, while it is difficult to use for the principle analysis of attention mechanisms. We want to design a simple and intuitive indicator to show the shift in energy cost and explore the underlying reasons. 

Spike counts of SNNs are positively correlated with energy cost and can be used to roughly measure energy cost. A spike corresponds to some AC operations, and how much is associated with the design of network architecture (i.e., $FL_{SNNConv}^{n}$ and $FL_{SNNFC}^{m}$ depend on architecture). Definitely, the lower the spike counts, the better the energy efficiency. To compute the spike counts, we define the network average spiking activity rate (NASAR) as follows: at time step $t$, a spiking network's spiking activity rate (NSAR) is the ratio of spikes produced over all the neurons to the total number of neurons in this step; then we define the NASAR which averages NSAR across all time steps $T$. Spike count is equal to \emph{NASAR} $\cdot$ \emph{neuron number} $\cdot$ \emph{T}. Once the network architecture and time step are determined, \emph{neuron number} and \emph{T} are fixed constants. So the spike counts are only related to \emph{NASAR}. We mainly exploit NASAR and NSAR to approximately represent energy costs and explore how the attention module brings the energy shift for SNNs.

\section{Experiments}\label{section:Experiment}
In this section, we investigate the \emph{effectiveness} and \emph{efficiency} of MA-SNN across a range of tasks, datasets, and architectures. We conduct extensive experiments on the datasets: DVS128 Gesture/Gait for event-based action recognition in Section~\ref{sec:exper_dvs128}; ImageNet-1K for static image classification in Section~\ref{subsec:exp_image}.

To perform better apple-to-apple comparisons, we first re-implementation all the reported performance of networks\cite{Fang_2021_ICCV,yao_2021_TASNN,Hu_2021_MS} in the PyTorch framework and set them as our baselines. When training the baseline (vanilla) models or MA-integrated models, we follow their training schemes (i.e., hyper-parameter settings), if not otherwise specified. Throughout all experiments, we verify that MA outperforms all the vanilla models in both performance and efficiency without bells and whistles, then the general applicability of MA across different architectures and as well as different tasks are demonstrated. We here directly give our recommended combination of attention dimensions. Ablation studies of attention design are given in Section~\ref{section:ablation_study}.

\begin{table*}[!ht]
	\renewcommand{\arraystretch}{1.2}
	\centering
	\caption{Comparison with previous works on Gesture and Gait. The numbers or percentages in brackets denote the performance improvement or reduction proportions of spiking activity over the re-implementation baseline \cite{yao_2021_TASNN} and \cite{Fang_2021_ICCV}. The format of ``A\std{a}'' represents the mean and variance of the accuracy of repeat ten independent experiments.}

\resizebox{\textwidth}{!}{
	\begin{tabular}{cccccccc}
	\toprule[1.2pt]
	
	\multirow{2}{*}{Datasets} & \multirow{2}{*}{Work} & \multirow{2}{*}{$dt \times T$}  & \multirow{2}{*}{Acc. (\%)} & \multirow{2}{*}{$r_{EE}$} & \multirow{2}{*}{Neuron number} & \multicolumn{2}{c}{Roughly Energy Evaluation}   \\ \cline{7-8}
	& & & & & & NASAR & Spike Counts ($\times 10^{6}$) \\
	\midrule[0.8pt]
	& Amir \emph{et al.} 2017\cite{amir_Gesture_dataset_2017} &$1 \times 120$ & 92.59 & - & - & - & -\\ 
	& Shrestha \emph{et al.} 2019\cite{Slayer} &$5 \times 300$ & 93.64 & - & - & - & -\\ 
	& Zheng \emph{et al.} 2021\cite{zheng_Going_Deeper_SNN_2021} &$30 \times 40$ & 96.87 & - & 229,376 & 0.146 & 1.337\\ 
DVS128	& Fang \emph{et al.} 2021\cite{fang_deep_SNN_2021} &$375 \times 16$ & 97.92 & - & 1,922,304 & 0.053 & 1.636\\ \cline{2-8}
	Gesture & Yao \emph{et al.} 2021\cite{yao_2021_TASNN} (Vanilla SNN) &\multirow{2}{*}{$15 \times 60$}  & 90.63\std{0.58} & 1$\times$ & 106,763 & 0.173 & 1.108\\
	 & + TCSA (\textbf{This Work})&  & \textbf{96.53\std{0.57} (+5.9)} & \textbf{3.4$\times$} & 106,763 & \textbf{0.026 (-84.9\%)} & \textbf{0.167 (-84.9\%)}\\ \cline{2-8}
	 & Fang \emph{et al.}\cite{Fang_2021_ICCV} (Vanilla SNN) &\multirow{2}{*}{$300 \times 20$} & 95.58\std{1.00} & 1$\times$ & 2,793,472 & 0.074 & 4.134  \\ 
	  & + TCSA (\textbf{This Work}) &  & \textbf{98.23 \std{0.46} (+2.7)} & \textbf{1.8$\times$} & 2,793,472 & \textbf{0.011(-85.1\%)} & \textbf{0.615(-85.1\%)} \\ \hline
    
    & Wang \emph{et al.} 2019\cite{wang_gait_cvpr_2019}   & $4400\times$1  & 89.9  & - & - & - &-\\  
	 & Wang \emph{et al.} 2021\cite{wang_Gait_PAMI_2021}  & $4400\times$1  &  94.9 & - & - &  - &-\\
    \cline{2-8} 
 
DVS128	& Yao \emph{et al.} 2021\cite{yao_2021_TASNN} (Vanilla SNN) &\multirow{2}{*}{$15 \times 60$}  & 87.59\std{0.47} & 1$\times$ & 106,772 & 0.245  & 1.570\\
	Gait &+ TCSA (\textbf{This Work})&  & \textbf{92.29\std{1.14} (+4.7)} & \textbf{3.2$\times$} & 106,772 & \textbf{0.045 (-81.6\%)} & \textbf{0.288 (-81.6\%)} \\\cline{2-8}
	  & Fang \emph{et al.}\cite{Fang_2021_ICCV} (Vanilla SNN) &\multirow{2}{*}{$220 \times 20$} & 89.87\std{1.32} & 1$\times$ & 2,793,472  & 0.051  & 2.849 \\ 
    & + TCSA (\textbf{This Work}) & & \textbf{92.78\std{0.79} (+2.9)} & \textbf{1.7$\times$} & 2,793,472 & \textbf{0.013(-74.5\%)} & \textbf{0.726(-74.5\%)} \\ 
	
	\bottomrule[1.2pt]
	\end{tabular}
	}
	\label{Table:Event_main_result}
\end{table*}

\subsection{Event-based Action Recognition}\label{sec:exper_dvs128}
\subsubsection{Experimental Setup}
\textbf{Datasets.} DVS128 Gesture\cite{amir_Gesture_dataset_2017} and DVS128 Gait \cite{wang_gait_cvpr_2019} are both event stream datasets captured by the DVS camera, which has the same \textmu s level temporal resolution and $128 \times 128$ spatial resolution but have different visual information. They capture the natural human motion, gestures and gaits. Gesture contains 11 kinds of hand gestures from 29 subjects under 3 kinds of illumination conditions, and there are 1,342 samples that each gesture has an average duration of 6 seconds. Gait contains various gaits from 21 volunteers (15 males and 6 females) under 2 kinds of viewing angles, and it records 4,200 samples where each gait has an average duration of 4.4 seconds.

\textbf{Learning.} We adopt the same experimental setup for vanilla SNNs and TCSA-SNN (here, we select TCSA as the MA), including network structure, hyper-parameters, learning methods, etc. MA is a plug-and-play module, all we have done is just add it to the vanilla models. We used \emph{identical experimental setup} to analyze the effectiveness and efficiency of the same network on different datasets. All learning methods follow \cite{yao_2021_TASNN}, and details are given in Section~S2.1. Two kinds of vanilla structures are performed for Gesture and Gait to evaluate the energy shift induced by different backbones. One has three Conv layers, following \cite{yao_2021_TASNN}. The other has five Conv layers, following \cite{Fang_2021_ICCV}. Moreover, we set various output latency $t_{lat}=dt \times T$ to investigate the effect of TCSA-SNN under the multi-scale constraints of latency. 

\begin{table*}[htbp]
\renewcommand{\arraystretch}{1.2}
\centering
\caption{Comparison with previous works on ImageNet-1K. * The input crops are enlarged to 288$\times$288 in inference.}
\begin{tabular}{cccccc}
\toprule[1.2pt]
Methods  & Work    & Model      & Time step ($T$)      & Top-1 Acc.(\%)  & Energy Efficiency ($r_{EE}$) \\ 
\midrule[0.8pt]
\multirow{10}{*}{ANN-to-SNN}  & \multirow{2}{*}{Sengupta \emph{et al.} 2019\cite{Senguputa_2013_going_deeper}} & VGG-16 & 2500 & 69.96 & -   \\ 
    &    & ResNet-34  & 2500    & 65.47   & - \\ 
    & \multirow{2}{*}{Han \emph{et al.} 2020\cite{Han_2020_CVPR}}    & VGG-16 & 4096  & 73.09  & -  \\
    &    & ResNet-34 & 4096  & 69.89   & - \\                              
    &  \multirow{2}{*}{ Wu \emph{et al.} 2021\cite{wu_2021_progressive}}    & AlexNet & 16  & 55.19  & -  \\  
    &    & VGG-16 & 16  & 65.08   & - \\ 
    & Li \emph{et al.} 2021\cite{li_2021_ann2SNN}    & VGG-16 &  \textbf{2048}  & \textbf{75.32}  & -  \\    
    & St{\"o}ckl \emph{et al.} 2021\cite{stockl2021optimized}    & ResNet-50 & 500  & 75.10 & -   \\
    &  Bu \emph{et al.} 2022\cite{bu_2022_optimized}   & VGG-16 & 512  & 74.69 & -   \\ \midrule
      
\multirow{6}{*}{Direct Training}   & \multirow{2}{*}{Zheng \emph{et al.} 2021\cite{zheng_Going_Deeper_SNN_2021}}  & ResNet-50       & 6    & 64.88   & - \\
    & & Wide-ResNet-34  & 6  & 67.05  & -  \\ 
    & \multirow{2}{*}{ Fang \emph{et al.} 2021\cite{fang_deep_SNN_2021}} & ResNet-34 & 4 & 67.04  & -  \\
     &        & ResNet-101      &      4  & 68.76   & - \\ 
    & Deng \emph{et al.} 2022\cite{deng2022temporal} & ResNet-34 & 4 & 68.00  & -  \\\cline{1-6}
\multirow{2}{*}{Direct Training}     & Hu \emph{et al.} 2021\cite{Hu_2021_MS} (SOTA backbone)  & ResNet-18 & 6 & 63.10   & 3.3$\times$ \\
     & \textbf{This Work (+CSA)}        & ResNet-18  &     1  & 63.97  & \textbf{29.7$\times$}  \\
Backpropagation     &   ANN     & ResNet-18 &     -  & 69.76    & 1$\times$ \\\midrule
     
\multirow{2}{*}{Direct Training}     & Hu \emph{et al.} 2021\cite{Hu_2021_MS} (SOTA backbone) & ResNet-34 & 6 & 69.42 & $ 3.8\times$\\
    & \textbf{This Work (+CSA)}  & ResNet-34       &   1  &  69.15  & \textbf{29.6$\times$}  \\
Backpropagation    & ANN  & ResNet-34  &  -   &  73.30  & 1$\times$ \\\midrule

\multirow{3}{*}{Direct Training}  & Hu \emph{et al.} 2021\cite{Hu_2021_MS} (SOTA backbone)  & ResNet-104*  & 5  &  76.02 & 5.3$\times$ \\
    &  \textbf{This Work (+CSA)} & ResNet-104*     &  \textbf{4}   &  \textbf{77.08}  & \textbf{7.4$\times$}  \\
    &  \textbf{This Work (+CSA)} & ResNet-104*     &  \textbf{1}   & \textbf{75.92}   & \textbf{31.8$\times$}   \\
Backpropagation    &  ANN\cite{Hu_2021_MS} & ResNet-104  & -  &  76.87 & 1$\times$  \\ 
\bottomrule[1.2pt]
\end{tabular}
\label{tab:ImageNet_results}
\end{table*}

\begin{table}[t]
	\renewcommand{\arraystretch}{1.2}
	\centering
	\caption{Comparison with baselines on ImageNet-1K (inference spatial resolution is $224\times224$). Based on the MS-Res-SNN\cite{Hu_2021_MS}, we re-implement various baseline structures and their attention counterpart with $T=1$ and report corresponding performance and energy shift. The case of $T=1$ can be regarded as pre-training of multi-time step SNNs\cite{lin_2022_pretrain}. Note that, compared with attention CNNs, the accuracy improvement of the attention on MS-Res-SNN is very significant, e.g., using identical CSA for Res-CNN-34\cite{CBAM} and Res-SNN-34 can improve the accuracy by +0.7 and +5.0 percent, respectively.}
	\begin{tabular}{cccc}
	\midrule[0.8pt]
	 Model ($T=1$) & Top-1 Acc. (\%) & $r_{EE}$ & NASAR\\
	\midrule[0.8pt]
	MS-Res-SNN-18\cite{Hu_2021_MS} & 61.70  & 1$\times$ & 0.224\\ 
	 + CA (\textbf{This work}) & 63.42 \textbf{(+1.7)} & 1.4$\times$ & 0.165 (-26.3\%)\\ 
	 + CSA (\textbf{This work}) & 63.97 \textbf{(+2.3)} & 1.5$\times$ & 0.148 (-33.9\%)\\ \hline
	 MS-Res-SNN-34\cite{Hu_2021_MS} & 64.13  & 1$\times$ & 0.203\\
	 + CA (\textbf{This work})  &67.96 \textbf{(+3.8)}& 1.2$\times$ & 0.167 (-17.7\%)\\ 
	 + CSA (\textbf{This work})  &69.15 \textbf{(+5.0)}& 1.3$\times$ & 0.153 (-24.6\%)\\\hline
	 MS-Res-SNN-104\cite{Hu_2021_MS} & 71.57  & 1$\times$ & 0.218\\
	 + CA (\textbf{This work}) & 72.03 \textbf{(+0.5)} &  1.1$\times$ & 0.195(-10.6\%)\\
	 + CSA (\textbf{This work}) & 73.82 \textbf{(+2.3)}&  1.1$\times$ & 0.201 (-7.8\%)\\
	\bottomrule[1.2pt]
	\end{tabular}
	\label{Table:result_CIFAR10_Image_result}
\end{table}

\subsubsection{Results and Analysis}
\textbf{Effectiveness.} In Table~\ref{Table:Event_main_result}, we report the accuracy of each vanilla model and its attention counterpart, and compare TCSA-SNN with previous works. We observe that in every comparison, TCSA-SNN outperforms the vanilla architectures, suggesting that the benefits of attention modules are not confined to a single event-based dataset, limited base architecture, or fixed output latency. TCSA-SNN performs better performance on the three-layer vanilla model (following \cite{yao_2021_TASNN}), which yields good gains of 5.9\% and 4.7\% on Gesture and Gait, respectively. The gains are consistent with the five-layer vanilla model (following \cite{Fang_2021_ICCV}). Thus, attention can effectively facilitate the performance of lightweight plain SNN models for event-based tasks.   

\textbf{Efficiency.} From Table~\ref{Table:Event_main_result}, we observe a significant improvement in inference efficiency between our TCSA-SNN and vanilla models. On Gesture, three-layer TCSA-SNN (following \cite{yao_2021_TASNN}) and five-layer TCSA-SNN (following \cite{Fang_2021_ICCV}) increase the energy efficiency by up to 3.4$\times$ and 1.8$\times$, respectively. We see a similar trend with regard to the effect of TCSA on Gait, finding that the TCSA-SNN outperforms our re-implemented counterpart baselines with both effectiveness and energy efficiency. Moreover, we observe that the NASAR is associated with the model size. Networks with more spiking neurons usually have sparser spiking activity, such as the NASAR of the three-layer and five-layer baseline on Gesture are 0.214 and 0.074, respectively. Incredibly, on Gesture, the NASAR of five-layer SNN drops to 0.011 with the help of TCA, which means that \emph{only 1.1\% spiking neurons are activated at each time step}.  

It should be noted that only when the benefits outweigh the costs, i.e., $E_{AC} \cdot \Delta_{AC} > E_{MAC} \cdot \Delta_{MAC}$, the energy cost can be dropped, and all of our attention designs in this paper meet this condition (see Section~\ref{subsec:exp_MA_SNN}). Moreover, the energy efficiency of attention SNNs is associated with the backbone structure ($E_{base}$). On Gait, the spike counts of the three-layer baseline are reduced by 81.6\% resulting in 3.2$\times$ better energy efficiency. By contrast, in the five-layer baseline, the reduction of spike counts and the value of $r_{EE}$ are 74.5\% and 1.7$\times$, respectively. The underlying reason is that there are more MAC operations in the first encoding layer of the five-layer baseline, which is induced by the architecture design and leads to higher $E_{base}$.

\subsection{Static Image Classification}\label{subsec:exp_image}
\subsubsection{Experimental Setup}
\textbf{Datasets.} ImageNet-1K\cite{deng2009imagenet} is the most typical static image dataset, which is widely used in the field of image classification. It provides a large-scale natural image dataset containing a total of 1,000 categories, which consists of 1.28 million training images and 50k test images. 


\textbf{Learning.} Similar to event-based recognition, each baseline architecture and its corresponding MA counterpart for ImageNet is trained with identical optimization schemes. We opted to use MS-ResNet-SNN\cite{Hu_2021_MS} architectures as strong baselines to assess the effectiveness and efficiency of attention modules and follow the training and evaluation protocols described in \cite{Hu_2021_MS}. All models are trained from re-implement scratch. All learning details are given in Section~S2.2 of SM. We exploit channel and spatial attention for ImageNet-1K, whose temporal information is weak.

\subsubsection{Results and Analysis}
\textbf{Effectiveness.} We would first set $T=1$ for all experiments which can be regarded as pre-training of multi-time step SNNs\cite{lin_2022_pretrain} and the results of top-1 accuracy on ImageNet-1K are reported in Table~\ref{Table:result_CIFAR10_Image_result}. We begin by comparing Att-Res-SNN against Res-SNN baselines with different depths. We observe that attention modules consistently improve performance across different depths. Remarkably, CSA-Res-SNN-34 exceeds Res-SNN-34 by +5.0 percent. By contrast, an identical CSA module used for Res-CNN-34 induces only +0.7 percent performance gain\cite{CBAM}.    
In Table~\ref{tab:ImageNet_results}, we make a comprehensive comparison with prior works on ImageNet-1K, including ANN-to-SNN, direct training SNN, and ANN. At the same network depth, we see that CSA-Res-SNN with single-time step can exceed or approach prior SOTA results\cite{Hu_2021_MS}, which uses multi-time steps. For example, CSA-Res-SNN-18, with an accuracy of 63.97\%, is better than the 63.10\% of its counterpart 6-time step backbone. The single-time step accuracy advantage of Res-SNN brought by the attention can also be extended to the deeper backbones (Res-SNN-34 and Res-SNN-104). Single-time step CSA-Res-SNN-34/104 has an accuracy of 69.15\%/75.92\%, which is slightly inferior to its 6/5-time step counterpart Res-SNN-34/104 backbone (69.42\%/76.02\%). We further extend CSA-Res-SNN-104 to 4-time steps and obtain an accuracy of 77.08\% on ImageNet-1K, which is the SOTA result in the SNN domain (including direct training and ANN-to-SNN), and better than the ANN with the same architecture(76.87\%). In contrast to SOTA ANN-to-SNN, our model have better accuracy (77.08\% vs. 75.32\%) and lower latency (4-time step vs. 2048-time step).

\textbf{Inference Efficiency.} From Table~\ref{Table:result_CIFAR10_Image_result} ($T=1$), we observe that employing the attention can drop the energy cost and improve performance simultaneously. For example, compared with Res-SNN-18, the NASAR of CSA-Res-SNN-18 can be reduced from 0.224 to 0.148, which means that each neuron emits only a 0.15 spike on average. Meanwhile, the task accuracy is improved from 61.70\% to 63.97\%, and the energy efficiency is up to 1.5$\times$. Different from the plain SNNs in event-based tasks, additional energy cost ($E_{MAC} \cdot \Delta_{MAC}$) caused by the attention module in Res-SNNs can be approximate ignored since the high $E_{base}$, e.g., 33.9\% decrease of spiking activity in Res-SNN-18 incurs 31.7\% reduction of energy cost ($1.5 \times$ better energy efficiency). In Table~\ref{tab:ImageNet_results}, we set the energy cost of ANN as the $E_{base}$ (1$\times$) and assess the energy cost of ANN/SNN with the same structures. As opposed to the Res-ANN-104, 5-time step Res-SNN-104 and 4-time step CSA-Res-SNN-104 achieve up to 5.3$\times$ and 7.4$\times$ better energy efficiency, respectively. Dramatically, single-times step CSA-Res-SNN-104 has up to 31.8$\times$ better compute energy efficiency compared to the Res-ANN-104. The results reported in Table~\ref{tab:ImageNet_results} show that deep SNNs have significant energy efficiency advantages over deep ANNs. Especially the single-step deep SNNs can amplify this advantage.

\textbf{Training Efficiency.} In this paper, we first train large-scale Att-Res-SNN by a single-time step and found that it can yield better or comparable accuracy than multi-time steps Res-SNN. Meanwhile, in terms of training efficiency, we hold a $6.4\times$ training acceleration and less hardware requirement (details in Section~S2.3 of SM). To the best of our knowledge, there are currently only two works\cite{fang_deep_SNN_2021,Hu_2021_MS} that directly train SNN with more than 100 layers on ImageNet-1K. Because the long training time and enormous hardware resources caused by large-scale and multi-time steps are prohibitive. Existing work attempts to alleviate this dilemma by designing complex training tricks\cite{chowdhury_2021_one_timstep_1} or complicated new spiking neuron\cite{suetake_2022_one_timestep_3}. Our experimental results show another simple potential solution to break the predicament, which includes two steps. First, the SNN community can focus on the single-step simulation algorithm design of large-scale SNNs, because it requires lower training time and less hardware. Our experimental results show that this is feasible. Attention can help single-step SNNs outperform or approach multi-step SNNs, which also demonstrates the great potential of SNNs. Then, the researchers separately focus on how to effectively extend single-step SNNs to multi-step SNNs, such as using the pre-train method\cite{lin_2022_pretrain}. We believe that reducing training time and hardware costs will help the development of SNN fields.

\begin{figure*}[!htbp]
\centering
\includegraphics[scale=0.52]{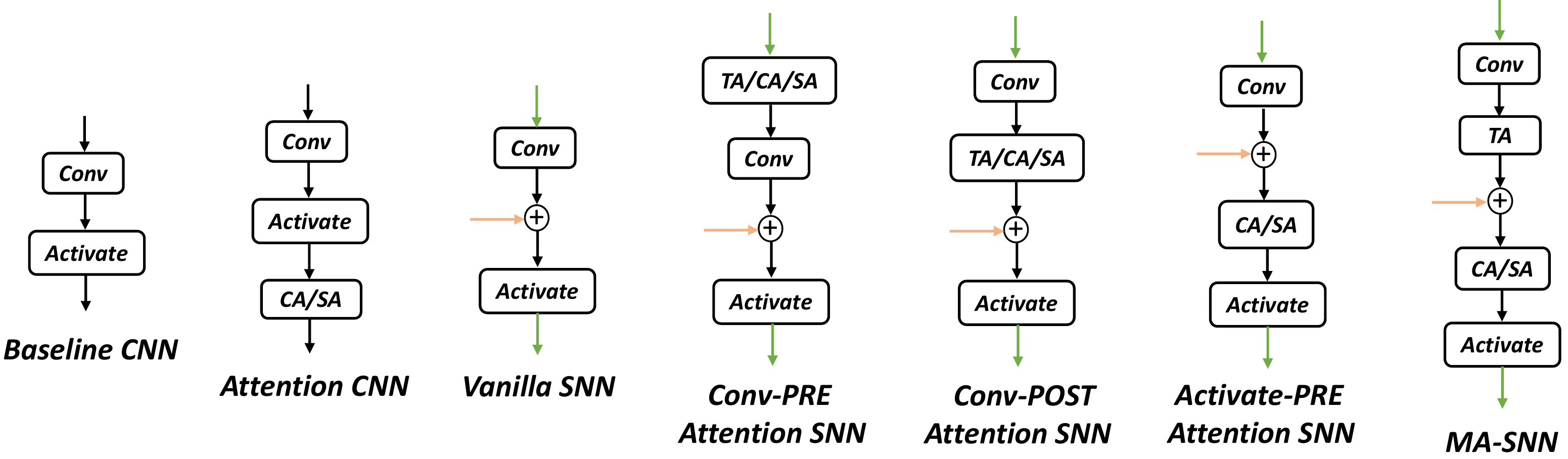}
\caption{Attention Locations in Plain SNN. From left to right: baseline CNN, processing spatial information. Attention CNN, performing attention module after activation\cite{CBAM,SE_PAMI}. Vanilla SNN, processing spatio-temporal information. Conv-PRE Attention SNN, inserting attention modules before the Conv operation. Conv-POST, acting attention modules after the Conv operation while before the spatio-temporal integration. Activate-PRE, executing attention modules on the integrated membrane potential. MA-SNN, our recommended attention locations of three dimensions in plain SNN.}
\label{Fig:MA_Location}
\end{figure*}

\begin{table}[!t]
	\renewcommand{\arraystretch}{1.2}
	\centering
	\caption{Effect of attention locations in three-layer SNN\cite{yao_2021_TASNN} at different dimensions on Gesture with $dt=15, T=60$}
	\begin{tabular}{cccc}
    	\toprule[1.2pt]
    	Model & Attention Location & Acc. (\%)  & NASAR \\
    	\midrule[0.8pt]
    	Vanilla SNN \cite{yao_2021_TASNN}& - & 90.63\std{0.58} & 0.174\\ \hline
    	\multirow{2}{*}{TA-SNN} & Conv-PRE & 91.01\std{0.69}& 0.074\\
    	 & Conv-POST & 92.60\std{0.47} & 0.073\\ \hline
    	\multirow{3}{*}{CA-SNN} &Conv-PRE & 91.84\std{0.60} & 0.086 \\
    	&Conv-POST & 91.58\std{0.29} & 0.097 \\
    	&Activate-PRE & \textbf{93.88\std{0.34}} & 0.072 \\ \hline
    	\multirow{3}{*}{SA-SNN}&Conv-PRE & 92.85\std{0.78} & 0.060 \\
    	&Conv-POST & 92.57\std{0.28} & 0.062 \\
    	&Activate-PRE & 92.50\std{0.47} & \textbf{0.058} \\
    	\bottomrule[1.2pt]
	\end{tabular}
	\label{Table:Attention_Location}
\end{table}

\section{Ablation Study}\label{section:ablation_study}
We conduct ablation experiments to gain a better understanding of the effectiveness and efficiency of adopting different configurations on the attention design for SNNs. All ablation experiments are performed by following the experimental setup in section \ref{section:Experiment}, if not otherwise specified. 

\subsection{Attention Locations in Plain SNNs}\label{subsec:attention_location}
Lightweight plain SNNs are suitable for real scenarios requiring latency, accuracy, and energy cost. We explore the effectiveness and spiking activity of SNNs with various attention locations based on the Gesture. The baseline architecture follows the three-layer SNN in \cite{yao_2021_TASNN}. The attention location design of SNNs is more sophisticated than CNNs. Firstly, the event-driven nature of SNNs will be destroyed if we copy the attention location in CNN (behind the activate operation, see Fig.~\ref{Fig:MA_Location}). Secondly, the goal of attention is to optimize the membrane potentials, which can be achieved by exploiting various attention dimensions separately or simultaneously. Thirdly, we need to examine the effect of the attention dimension on its location. Comprehensively, we consider three possible attention locations: (1) Conv-PRE, in which the attention module is inserted before the Conv operation; (2) Conv-POST, in which the attention module is moved between the Conv and the integration operation and (3) Activate-PRE, in which the attention unit acts on the integrated membrane potential. These variants are illustrated in Fig.~\ref{Fig:MA_Location}. Individual analyses and assessments are performed on these variants according to attention dimensions.

\textbf{TA Location.} Firstly, TA location cannot be Activate-PRE, i.e., TA cannot work on the membrane potential which has already aggregated spatio-temporal information, because it is impossible to recalibrate the state that has occurred. Performance of Conv-PRE and Conv-POST for TA are reported in rows 3 and 4 of Table~\ref{Table:Attention_Location}. We observe that both Conv-PRE and Conv-POST can improve performance and drop NASAR. The usage of the Conv-POST leads to better performance with a similar NASAR to Conv-PRE. So we select Conv-POST as the TA location.

\textbf{CA Location.} All location variants can be used to insert CA. We assess these locations and report results in Table~\ref{Table:Attention_Location}. We see that incorporating the Activate-PRE CA for SNNs can obtain the best performance and lowest NASAR. This suggests that using CA to directly optimize spatial-temporal fused membrane potential is more effective and efficient than only optimizing spatial features. 

\textbf{SA Location.} We finally explore the effect of using SA alone. The results are reported in the last three rows of Table~\ref{Table:Attention_Location}. We observe that individual SA can also improve performance and drop NASAR without sensitivity to the location. In the practice of attention CNNs, SA is usually performed together with CA and follows CA\cite{CBAM}. We adopt this serial setting in this paper and set the Activate-PRE location for SA. 

The comparison in Table~\ref{Table:Attention_Location} shows that effectiveness and efficiency are \emph{robust} at various dimensions to a range of attention locations. Performance always goes up no matter which location is chosen. While specific gains of accuracy and NASAR vary from location to location. These experimental results also confirm our point that optimized membrane potential can induce sparser spiking activity and better task performance simultaneously. Comprehensively considering effectiveness and efficiency, our recommended attention locations are shown at the rightmost in Fig.~\ref{Fig:MA_Location}.

\subsection{Different Attention Variants}\label{subsec:choice_of_attention_module}
Attention modules generally concern three metrics for model assessment and comparison: accuracy, additional computational burden, and parameters. Many attention modules are designed to meet various metrics limitations in different scenes. Here we investigate three typical attention modules and apply them to SNNs. One of the representative works is \emph{CBAM}\cite{CBAM}, which stacks channel and spatial attention in series to achieve higher accuracy. To reduce the computational burden of the two-layer FC operation in classic CBAM, \emph{ECANet}\cite{Wang_2020_ECA} uses a 1D convolution to model the interaction between channels. In addition, attention can be viewed as the process of feature optimization, e.g., \emph{SimAM}\cite{SimAM} proposes an attention module based on mathematics and neuroscience theories with parameter-free. We explore the influence of the above three modules by directly integrating them into three-layer baseline SNN\cite{yao_2021_TASNN}, where we adopt the same Activate-PRE location. 

The metrics of attention SNNs are compared to vanilla SNNs in Table~\ref{Table:Different_Attention}. We note that CBAM, ECA, and SimAM all perform well to improve performance and simultaneously lead to sparser spiking activity. CBAM brings the highest performance gains. SimAM has no additional parameters and obtains a maximum additional computational burden. ECA has the smallest $\Delta_{MAC}$ and negligible additional parameters. The NASARs of these attention SNNs are close. This experiment suggests that the benefits of the performance improvements and sparser spiking activity produced by optimizing membrane potentials are fairly \emph{robust} to different attention mechanisms. In practice, we can choose the attention module according to the requirement of specific scenarios. To achieve better performance, we use CBAM as the basis attention module of SNNs in this paper.  

\begin{table}[!t]
	\renewcommand{\arraystretch}{1.2}
	\centering
	\caption{Effect of Different attention modules in three-layer SNN\cite{yao_2021_TASNN} on Gesture with $dt=15, T=60$}
	\setlength{\tabcolsep}{1.3mm}{
	\begin{tabular}{ccccc}
	\toprule[1.2pt]
	Design & Acc. (\%) & Params ($\uparrow$) & NASAR  & $\Delta_{MAC}$ ($\uparrow$)\\
	\midrule[0.8pt]
	Baseline SNN\cite{yao_2021_TASNN} & 90.63\std{0.58} & - & 0.173  & - \\ \hline
	CBAM\cite{CBAM}-SNN & \textbf{93.88\std{0.34}} & +4,608  & 0.072 & +276,480 \\
	ECA\cite{Wang_2020_ECA}-SNN & 92.81\std{0.40} & +48 & 0.071  & +76,800 \\
	SimAM\cite{SimAM}-SNN & 91.98\std{0.17} & 0 & 0.083    & +25,559,043 \\
	\bottomrule[1.2pt]
	\end{tabular}
	}
	\label{Table:Different_Attention}
\end{table}

\subsection{Combinations of Attention Dimension}\label{subsec:exp_MA_SNN}
We perform an ablation study to assess the effectiveness and energy cost of the combination of various attention dimensions when executing them together. Results are reported in Table~\ref{Table:Attention_combination}. Simple superposition of various attention dimensions can further boost the representation power and energy efficiency of SNNs, e.g., three-dimensional attention TCSA-SNN achieves the highest accuracy gain of +5.9 and the high energy efficiency of 3.4$\times$ than vanilla SNN. The performance of TCSA-SNN also exceeds all two-dimensional or single-dimensional attention SNNs. Similarly, effectiveness and efficiency are \emph{robust} at various attention dimension combinations.

\begin{table}[!t]
	\renewcommand{\arraystretch}{1.2}
	\centering
	\caption{Effect of attention dimensions in three-layer Conv-based SNN\cite{yao_2021_TASNN} on Gesture with $dt=15, T=60$}
	\setlength{\tabcolsep}{1.3mm}{
	\begin{tabular}{cccc}
	\toprule[1.2pt]
	Model & Acc. (\%) & $r_{EE}$ & NASAR \\
	\midrule[0.8pt]
	Vanilla SNN\cite{yao_2021_TASNN} & 90.63\std{0.58} & 1$\times$ & 0.173  \\ \hline
	+ TA & 92.60\std{0.47} (+2.0) & 2.2$\times$ & 0.073 \\
	+ CA & 93.88\std{0.34} (+3.3)& 2.2$\times$ & 0.072 \\
	+ SA & 92.50\std{0.47} (+1.9)& 2.5$\times$ & 0.058 \\ \hline
	+ TCA & 95.73\std{0.49} (+5.1)& \textbf{3.5}$\times$ & 0.034 \\
	+ TSA & 94.31\std{0.50} (+3.7)& 3.4$\times$ & 0.030 \\
	+ CSA & 94.79\std{0.82} (+4.2)& 3.4$\times$ & 0.030  \\ \hline
	+ TCSA & \textbf{96.53\std{0.57} (+5.9)} & 3.4$\times$ & \textbf{0.026} \\
	\bottomrule[1.2pt]
	\end{tabular}
	}
	\label{Table:Attention_combination}
\end{table}

\begin{table}[!t]
	\renewcommand{\arraystretch}{1.2}
	\centering
	\caption{Effect of Different residual attention locations in Res-SNNs with $T=1$ on ImageNet-1K}
	\setlength{\tabcolsep}{1.3mm}{
	\begin{tabular}{ccc}
	\toprule[1.2pt]
	Model & Acc. (\%) & NASAR \\
	\midrule[0.8pt]
	Res-SNN-18\cite{Hu_2021_MS} & 61.70 & 0.224  \\ 
	CSA-Res-SNN-18-1 & \textbf{63.97(+2.3)} & 0.148  \\
	CSA-Res-SNN-18-2 & 63.49(+1.8) & 0.137  \\\hline
	Res-SNN-34\cite{Hu_2021_MS} & 64.13 & 0.203  \\ 
	CSA-Res-SNN-34-1 & \textbf{69.15(+5.0)} & 0.153  \\
	CSA-Res-SNN-34-2 & 68.36(+4.2) & 0.131  \\
	\bottomrule[1.2pt]
	\end{tabular}
	}
	\label{Table:Residual_Attention}
\end{table}

\begin{figure*}[!t]
\centering
\includegraphics[scale=0.51]{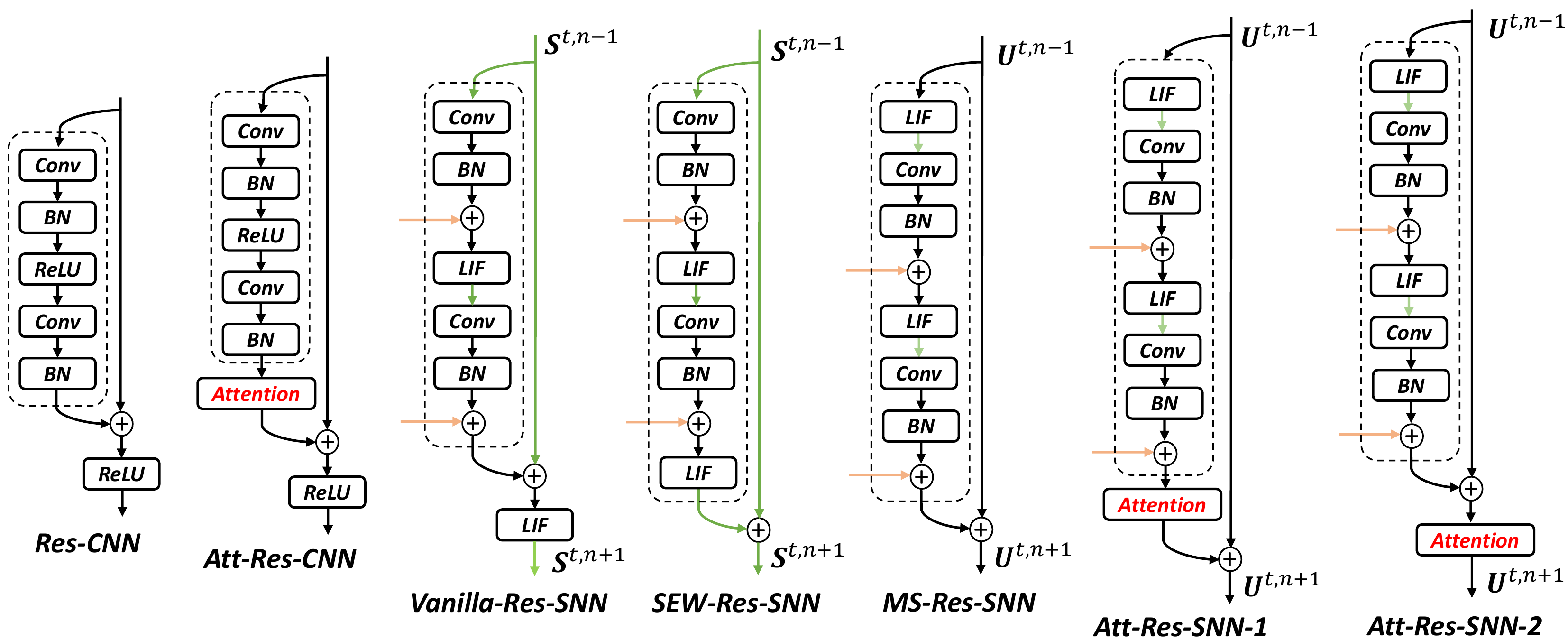}
\caption{Attention Residual Learning for SNNs. From left to right: Res-CNN\cite{he_resnet_2016}. Classic Att-Res-CNN\cite{SE_PAMI,CBAM}. Vanilla Res-SNN\cite{zheng_Going_Deeper_SNN_2021}, executing the same shortcut and residual block as Res-CNN. SEW-Res-SNN\cite{fang_deep_SNN_2021}, mainly building shortcuts between spikes from different layers. MS-Res-SNN\cite{Hu_2021_MS}, constructing a shortcut among membrane potential of spiking neurons in different layers. Att-Res-SNN-1 (our recommended method) and Att-Res-SNN-2, both select MS-Res-SNN as the backbone. The former performs attention between the residual block and shortcut connection, which is the same as classic Att-Res-CNN. The latter executes attention after the shortcut connection.}
\label{Fig:attention_res_SNN}
\end{figure*}

\subsection{Attention Residual Learning SNNs}\label{subsec:choice_of_attention_residual}

As discussed in Section~\ref{subsec:attention_residual_SNN}, Vanilla Res-SNN\cite{zheng_Going_Deeper_SNN_2021}, SEW-Res-SNN\cite{fang_deep_SNN_2021} and MS-Res-SNN\cite{Hu_2021_MS} are the only three kinds of residual learning of SNN by direct training, which dedicate to conquering the degradation problem of deep SNNs. Referring to our previous works\cite{fang_deep_SNN_2021,Hu_2021_MS}, a deeper model should have a training error no greater than its shallower counterpart if the added layers implement the identity mapping. Vanilla Res-SNN copies the experience of Res-CNN but can not obtain identity mapping since it makes a mismatching shortcut connection on membrane potential and spikes. The underlying reason is that CNNs activate analog values while SNNs activate spikes. To address this, SEW-Res-SNN and MS-Res-SNN respectively establish shortcuts between spikes or membrane potentials of different layers, where both can obtain identity mapping (see Fig.~\ref{Fig:attention_res_SNN}).   

In fact, we think the shortcut connection in MS-Res-SNN is identical to our motivation for introducing the attention, which can also be seen as a way to optimize the membrane potentials. Furthermore, given that MS-Res-SNN has higher accuracy, we choose it as the backbone model. In addition to the proposed attention residual design that performs attention between the residual block and shortcut connection, we consider another variant, Att-Res-SNN-2, in which the attention is moved after the shortcut. These variants are illustrated in Fig.\ref{Fig:attention_res_SNN} and the performance of each variant is reported in Table~\ref{Table:Residual_Attention}. We observe that both Att-Res-SNN-1 (our recommended method) and Att-Res-SNN-2 perform well (i.e., have \emph{robustness}) on effectiveness and efficiency concretely. Furthermore, both of them can overcome the degradation problem in general deep SNNs, i.e., they can achieve dynamical isometry (details in Section~\ref{subsec:Gradient_evolvement}). Moreover, Att-Res-SNN-1 is better in the effectiveness aspect, and Att-Res-SNN-2 has sparser spiking activity. Although it is beyond the scope of this work, we anticipate that further effectiveness and efficiency gains will be achievable simultaneously by tailoring backbone SNNs and attention module usage for specific tasks.

\section{Understanding and Visualizing Attention}\label{section:visualizing_attention}
We have shown that attention can concurrently boost the effectiveness and efficiency of a plain or deep SNN for various vision tasks. Here we provide an in-depth analysis of how an MA-integrated model (i.e., three-layer SNN + TCSA, Res-SNN-34 + CSA) may differ from its vanilla counterpart (i.e., three-layer SNN, Res-SNN-34). We first explore the gradient norm equality of attention deep SNNs by using the dynamical isometry framework proposed by Chen \emph{et al.}\cite{chen2020comprehensive}. Next, we provide visualization results of that the MA-SNN success to correctly classifying but the vanilla model fails. Then, we explore the effectiveness and efficiency of MA-SNN by the proposed spiking response visualization method. Finally, we investigate the change in spiking activity rate induced by attention. 


\subsection{Gradient Evolvement in Att-Res-SNNs}\label{subsec:Gradient_evolvement}

In recent years, dynamical isometry has been developed as a theoretical explanation of well-behaved neural networks. When a deep neural network is dynamical isometry, it can avoid gradient vanishing or explosion and every singular value of its input-output Jacobian matrix remains close to one. In this subsection, we analyze that both of Att-Res-SNN-1 and Att-Res-SNN-2 in Fig.~\ref{Fig:attention_res_block} can achieve gradient norm equality with the help of block dynamical isometry framework\cite{chen2020comprehensive}. In a nutshell, Att-Res-Net-1 and Att-Res-Net-2 can avoid the drawback of degradation problem and attain great stability constituting a much shallower network in effect than it appears to be for gradient norm. 

Without loss of generality, a neural network can be viewed as a serial of blocks:
\begin{equation}
    f(x_0)=f^L_{\boldsymbol{\theta}^L}\circ f^{L-1}_{\boldsymbol{\theta}^{L-1}} \circ \cdots \circ f^1_{\boldsymbol{\theta}^1}(x_0),
    \label{eq:serial}
\end{equation}
where $\boldsymbol{\theta}^i$ is the parameter matrix of the $i$-th layer. For simplicity, we denote $\frac{\partial f^{j}}{\partial f^{j-1}}$ as $\boldsymbol{J}_j$. Let $\phi(\boldsymbol{J})$ be the expectation of $tr(\boldsymbol{J})$, and $\varphi(\boldsymbol{J})$ be $\phi(\boldsymbol{J}^2)-\phi(\boldsymbol{J})^2$. 

\begin{definition}[Block Dynamical Isometry] (Definition 3.1 in \cite{chen2020comprehensive})
Consider a neural network that can be represented as Eq.~\ref{eq:serial} and the $j$-th block’s Jacobian matrix is denoted as $\boldsymbol{J}_j$. If $ \forall$ j, $\phi(\boldsymbol{J}_j {\boldsymbol{J}^T_j})$ $\approx 1$ and $\varphi(\boldsymbol{J}_j {\boldsymbol{J}^T_j}) \approx 0$, the network achieves block dynamical isometry.
\label{def:block dynamical isometry}
\end{definition}

\begin{lemma}[Shallow Network Trick] (Proposition 5.8 in \cite{chen2020comprehensive})
Assuming that for each of $L$ sequential blocks in a neural network, we have $\phi(\boldsymbol{J}_j {\boldsymbol{J}^T_j})=\omega+\tau\phi(\widetilde{\boldsymbol{J}_j}\widetilde{{\boldsymbol{J}_j}}^T)$ where $\boldsymbol{J}_j$ is its Jacobian matrix, $\omega$ and $\tau$ are two constants determined by the network structure. For a shallow $\lambda$-layer network, given $\lambda \in \mathbb{N} ^+ <L$, if $C_L^\lambda (1-\omega)^\lambda$ and $C_L^\lambda\tau^\lambda$ are small enough, the $L$-block network would be as stable as a $\lambda$-layer network when both networks have $\forall j $, $\phi(\boldsymbol{J}_j {\boldsymbol{J}^T_j}) \approx 1$.
\label{lemma:shallow network trick}
\end{lemma}

\begin{table}[!t]
	\renewcommand{\arraystretch}{1.2}
	\centering
	\caption{$\phi(\boldsymbol{JJ}^T)$ and $\varphi(\boldsymbol{JJ}^T)$ of ReLU, Conv, Orthogonal, and sigmoid operations in neural networks. Results are collected from Lemma~S3 and Lemma~S4 in Section~S3.}
	\begin{tabular}{ccc}
    	\toprule[1.2pt]
    	Part & $\phi(\boldsymbol{JJ}^T)$ & $\varphi(\boldsymbol{JJ}^T)$ \\\hline

    	ReLU & $p$ & $p-p^2$    	\\ \hline    	
    	Conv  & $c_{in}k_{h}k_{w}\epsilon^2$  &- \\ \hline
        Orthogonal & 	${\gamma}^2$  & 0 	\\ \hline
        Sigmoid & 	$\frac{1}{16}$  & 0 	\\
    	\bottomrule[1.2pt]
	\end{tabular}
	\label{Table:common_prove_part}
\end{table}

Based on the definition of block dynamical isometry (Definition~\ref{def:block dynamical isometry}) and the shallow network trick (Lemma~\ref{lemma:shallow network trick}), we can judge whether Att-Res-SNN-1\&2 can achieve gradient norm equality or not.The serial and parallel connections in neural networks are denoted as Lemma~S1 (multiplication theorem) and Lemma~S2 (addition theorem) in Section~S3 of SM. Some commonly used network components in Att-Res-SNNs are summarized in Table~\ref{Table:common_prove_part} (Details are in Lemma~S3 and Lemma~S4 of Section~S3). 

\begin{theorem}[Gradient Norm Equality of Att-Res-SNNs]\label{Thm_grad}
Assuming two kinds of Att-Res-SNN designs in Fig.\ref{Fig:attention_res_block} consisting of $L$ sequential blocks, they can both achieve block dynamical isometry that Att-Res-SNNs could be as stable as a $\lambda$-layer network which satisfies $\phi(\boldsymbol{J}_j {\boldsymbol{J}^T_j}) \approx 1$ and $\lambda\in \mathbb{N^+} <L$.
\end{theorem}

\begin{proof}
According to Lemma~S1, the whole network's Jacobian matrix can be decomposed into the multiplication of its blocks’ Jacobian matrices. After integrating the attention module into the basic Res-SNN block, we expect that each Att-Res-SNN block should satisfy $\phi(\boldsymbol{J}_j {\boldsymbol{J}^T_j})$ $\approx 1$, which provides stable gradient evolvement. We first analyze the CA and SA blocks. Then we evaluate the CSA block consisting of a serial connection of CA and SA. Finally, we discuss Att-Res-SNN-1\&2 that integrate CSA blocks. 

The Jacobian matrix of channel (function $g_{c}(\cdot)$) and spatial (function $g_{s}(\cdot)$) attention block are denoted as $\boldsymbol{J}_{CA}$ and $\boldsymbol{J}_{SA}$, respectively. Assuming $\boldsymbol{W}_{c1}^{n}$ and $\boldsymbol{W}_{c0}^{n}$ in Eq.~\ref{eq:CA_1} are satisfy the Haar orthogonal initialization method. According to Eq.~\ref{eq:CA_1}, Eq.~\ref{eq:SA_1}, and Lemma~S1, $\phi(\boldsymbol{J}_{CA} {\boldsymbol{J}_{CA}^T})$ and $\phi(\boldsymbol{J}_{SA} {\boldsymbol{J}_{SA}^T})$ can be decomposed as follow:

\begin{equation}
\begin{aligned}
\phi(\boldsymbol{J}_{CA} {\boldsymbol{J}_{CA}^T})= 2 \phi(\boldsymbol{J}_{Sig} {\boldsymbol{J}_{Sig}^T})\phi(\boldsymbol{J}_{{W}_{c1}} {\boldsymbol{J}_{{W}_{c1}}^T}) \\ \phi(\boldsymbol{J}_{ReLU} {\boldsymbol{J}_{ReLU}^T})\phi(\boldsymbol{J}_{{W}_{c0}} {\boldsymbol{J}_{{W}_{c0}}^T}).
\end{aligned}
\end{equation}
and
\begin{equation}
\phi(\boldsymbol{J}_{SA} {\boldsymbol{J}_{SA}^T})=\phi(\boldsymbol{J}_{Sig} {\boldsymbol{J}_{Sig}^T})\phi(\boldsymbol{J}_{Conv} {\boldsymbol{J}_{Conv}^T}).
\end{equation}
According to Table~\ref{Table:common_prove_part}, we have
\begin{equation}
\phi(\boldsymbol{J}_{CA} {\boldsymbol{J}_{CA}^T})=\frac{\gamma_{w1}^2\gamma_{w0}^2p}{8},
\label{Eq:SA_J}
\end{equation}
and
\begin{equation}
\phi(\boldsymbol{J}_{SA} {\boldsymbol{J}_{SA}^T})=\frac{c_{in}k_{h}k_{w}\epsilon^2}{16}.
\label{Eq:CA_J}
\end{equation}

Then we evaluate the CSA block. For simplicity, we suppose the input of $g_{c}(\cdot)$ is $U$, the output of channel refinement is $U_{c}$. The channel refinement of membrane potentials can be described as
\begin{equation}
U_c = g_c(U) \odot U = g_c(U) (I U),
\end{equation}
where $I$ is the identity tensor, and we have
\begin{equation}
\frac{\partial U_c}{\partial U} = I g_c(U) + \frac{\partial g_c}{\partial U} (I U).
\end{equation}
Similarly, for spatial refinement of membrane potentials, we have
\begin{equation}
\begin{aligned}
U_s = g_s(U) \odot U = g_s(U) (I U),
\end{aligned}
\end{equation}
and
\begin{equation}
\frac{\partial U_s}{\partial U} = I g_s(U) + \frac{\partial g_s}{\partial U} (I U),
\end{equation}
where $U_s$ is the output of spatial refinement. Linking channel and spatial attention in tandem
\begin{equation}
U_{cs} = g_s(g_c(U) \odot U) \odot (g_c(U) \odot U),
\end{equation}
where $U_{cs}$ is the output of channel-spatial refinement. The Jacobian matrix of channel-spatial attention block $\boldsymbol{J}_{CSA}$ is
\begin{equation}
\begin{aligned}
\boldsymbol{J}_{CSA} &= \frac{\partial U_{cs}}{\partial U} = \frac{\partial U_{cs}}{\partial U_c} \frac{\partial U_c}{\partial U}\\
&=(I g_s(U_c) + \frac{\partial g_s}{\partial U_c} (I U_c))(I g_c(U) + \frac{\partial g_c}{\partial U} (I U)).
\end{aligned}
\end{equation}
Based on Lemma~S1 and Table~\ref{Table:common_prove_part}, we have
\begin{equation}
\begin{split}
&\phi(\boldsymbol{J}_{CSA} {\boldsymbol{J}_{CSA}^T})=\\
&\phi((I g_s(U_c) + \frac{\partial g_s}{\partial U_c} (I U_c)) (I g_s(U_c) + \frac{\partial g_s}{\partial U_c} (I U_c))^T)\\
&\phi((I g_c(U) + \frac{\partial g_c}{\partial U} (I U)) (I g_c(U) + \frac{\partial g_c}{\partial U} (I U))^T)\\
&=(\phi(\boldsymbol{J}_{SA} {\boldsymbol{J}_{SA}^T})
\phi((I U_c) (I U_c)^T)+
\phi(I g_s(U_c) (I g_s(U_c))^T))\\
&(\phi(\boldsymbol{J}_{CA} {\boldsymbol{J}_{CA}^T})
\phi((I U) (I U)^T)+
\phi(I g_c(U) (I g_c(U))^T)).
\end{split}
\label{Eq:CSA_J}
\end{equation}

According to Table~\ref{Table:common_prove_part}, the means of outputs from sigmoid function is $\frac{1}{2}$. Thus, $\phi(I g_c(U) (I g_c(U))^T)$, $\phi(I g_s(U_c) (I g_s(U_c))^T)$, and $\phi((I U_c) (I U_c)^T)$ are equal to $=\frac{1}{4}$. The means of $U$ are 0 and $C$ (a constant) in Att-Res-SNN-1 (BN output) and Att-Res-SNN-2 (BN output plus attention output), respectively. Bring Eq.~\ref{Eq:CA_J} and Eq.~\ref{Eq:SA_J} into Eq.~\ref{Eq:CSA_J}, then for Att-Res-SNN-1:
\begin{equation}
\begin{aligned}
\phi(\boldsymbol{J}_{CSA} {\boldsymbol{J}_{CSA}^T}) &=\\
& (\frac{1}{4}+\frac{c_{in}k_{h}k_{w}\epsilon^2}{16}\frac{1}{4})(\frac{1}{4}+ \frac{\gamma_{w1}^2\gamma_{w0}^2p}{8}\times 0),
\end{aligned}
\end{equation}
and for Att-Res-SNN-2:
\begin{equation}
\begin{aligned}
\phi(\boldsymbol{J}_{CSA} {\boldsymbol{J}_{CSA}^T}) &=\\
 & (\frac{1}{4}+\frac{c_{in}k_{h}k_{w}\epsilon^2}{16}\frac{1}{4})(\frac{1}{4}+ \frac{\gamma_{w1}^2\gamma_{w0}^2p}{8}\times C).
\end{aligned}
\end{equation}
Thus, for attention blocks (CA, SA, CSA), $\phi(\boldsymbol{J J}^T)=1$ can be achieved when $\epsilon$, $\gamma_{w1}^2$, $\gamma_{w0}^2$, and $p$ are set appropriately. 

Finally, we consider the Att-Res-SNN-1 and Att-Res-SNN-2 that Res-SNN block and CSA block form the new attention residual block in Fig~\ref{Fig:attention_res_block}. For simplicity, we denote the Jacobian matrix of attention residual block and shortcut as $\boldsymbol{J}_j$ and $\widetilde{\boldsymbol{J}_j}$ respectively. According to Lemma~\ref{lemma:shallow network trick}, for Att-Res-SNN-1, we have $\phi(\boldsymbol{J}_j {\boldsymbol{J}^T_j})=1+\zeta^2\phi(\boldsymbol{J}_{CSA} {\boldsymbol{J}_{CSA}^T})\phi(\widetilde{\boldsymbol{J}_j}\widetilde{{\boldsymbol{J}_j}}^T)$ where $\zeta$ is from the linear scale transformation $\zeta x + \beta$ within the normalization
at the bottom of the basic Res-SNN block. Att-Res-SNN-1 can be viewed as an extreme example of Lemma~\ref{lemma:shallow network trick} with $(1 - \omega) \rightarrow 0$. Therefore $\forall \lambda$, $C_L^\lambda (1 - \omega)^\lambda$ is close to zero, and $C_L^\lambda\zeta^\lambda$ can be small
enough for a given $\lambda$ if $\zeta$ is initialized as a relative small value. In this way, the non-optimal block's error will be influential only within $\lambda$ layers, and the Att-Res-SNN-1 will be as stable as a much shallower $\lambda$-layer network. For Att-Res-SNN-2, we can obtain $\phi(\boldsymbol{J}_j {\boldsymbol{J}^T_j})=\phi(\boldsymbol{J}_{CSA} {\boldsymbol{J}_{CSA}^T})(1+\zeta^2\phi(\widetilde{\boldsymbol{J}_j}\widetilde{{\boldsymbol{J}_j}}^T))$. So if the basic residual block can achieve dynamical isometry, Att-Res-SNN-2 can also do it.
\end{proof}

\begin{figure}[!t]
\centering
\subfigure[Case study on DVS128 Gesture]{\includegraphics[scale=0.52]{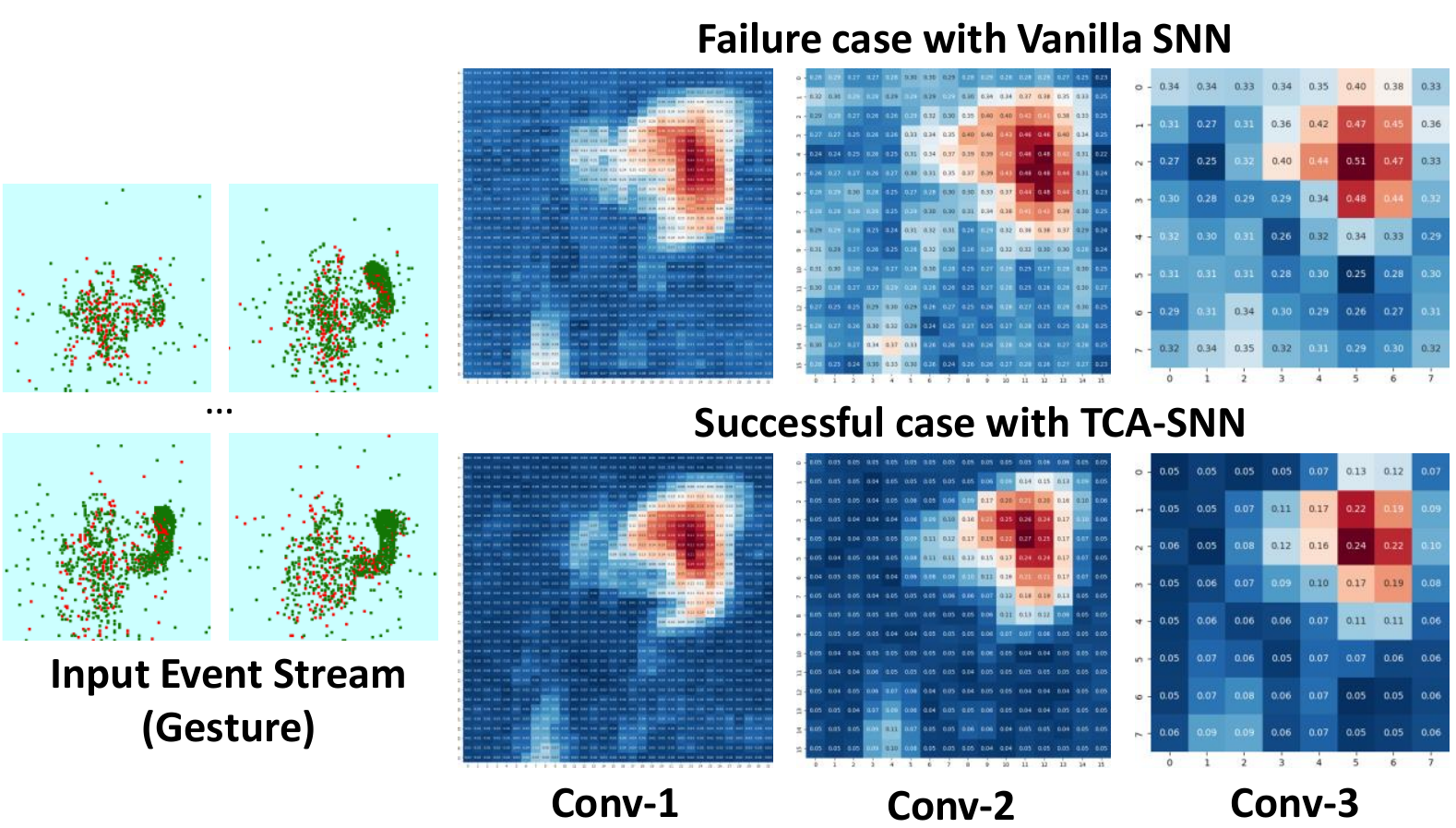}}
\subfigure[Case study on DVS128 Gait]{\includegraphics[scale=0.52]{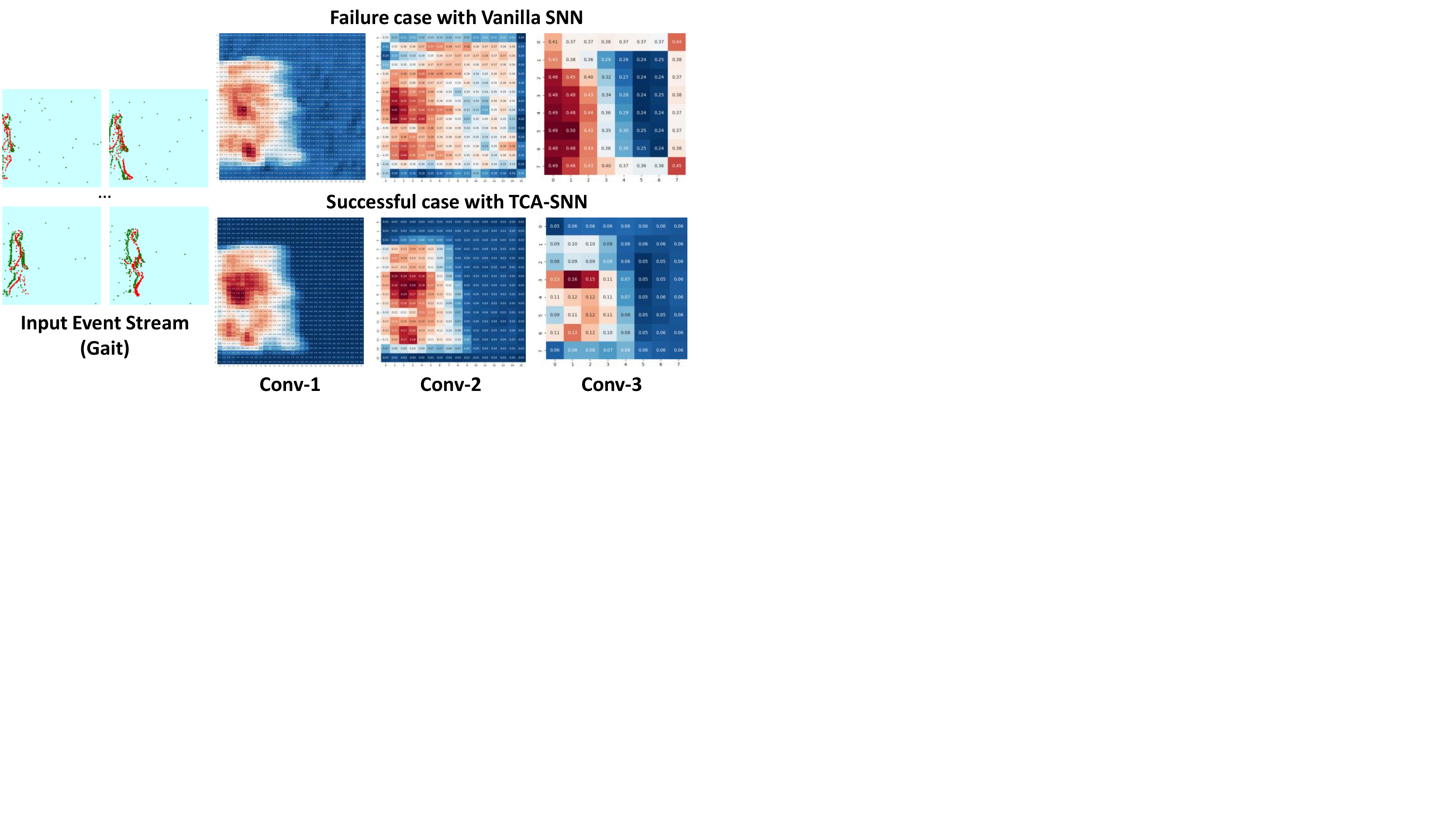}}
\caption{Case study on event-based tasks. We can observe that attention drives SNNs to focus on the target while the vanilla model shows more decentralized spiking activations. }
\label{Fig:Case_error}
\end{figure}

\subsection{Case Study}
Understanding attention mechanisms by visualizing intermediate features\cite{park_2020_attention_bam} or attention heat maps\cite{zhou_2016_CAM} of a single sample are the two most common methods in CNNs. Activation values can directly generate the former, and the latter is produced by class activation mapping (CAM)\cite{zhou_2016_CAM}. Compared with these traditional methods, SNN's visualization is more natural since it has only two active statuses. In this paper, for a single event-based sample, we averaged all the 4D ($[T, C, H, W]$) spiking maps of SNN into a 2D map ($[H, W]$) over the temporal and channel dimension at each layer. Then we plot the 2D feature, which represents the average spiking response of every layer for this sample. 

To visualize the effectiveness and efficiency of attention SNNs, we select two examples with regard to the case of the vanilla SNN failing in recognition but the attention SNN succeeds, where one from Gesture and the other from Gait. As shown in Fig.~\ref{Fig:Case_error}, each feature indicates the average spiking response of a layer of SNN. We make the following three observations about the effect of attention on SNN. First, the spiking activity is more concentrated in TCA-SNN, i.e., the red area of TCA-SNN is smaller and more focused. This suggests that attention is good for focusing on the important spatial region of intermediate channels. The second observation is that in the background region, the spiking response is suppressed. We see that attention darkens the color of the light blue area (background). The bluer the pixel, the closer the spiking activity rate is to 0. Finally, in all network layers, besides the obvious focus and suppression phenomena between vanilla and attention SNNs, we see a decrease in the overall spiking response induced by attention. For example, in Conv-3 of vanilla SNN and TCA-SNN on Gesture, the highest values are 0.51 and 0.24, respectively. These observations are consistent with the NASAR values of vanilla and attention SNNs in Table~\ref{Table:Attention_combination}. Thus, by optimizing the membrane potential of spiking neurons, attention induces sparser spiking activity of SNNs.

\begin{figure}[!t]
\centering
\includegraphics[scale=0.28]{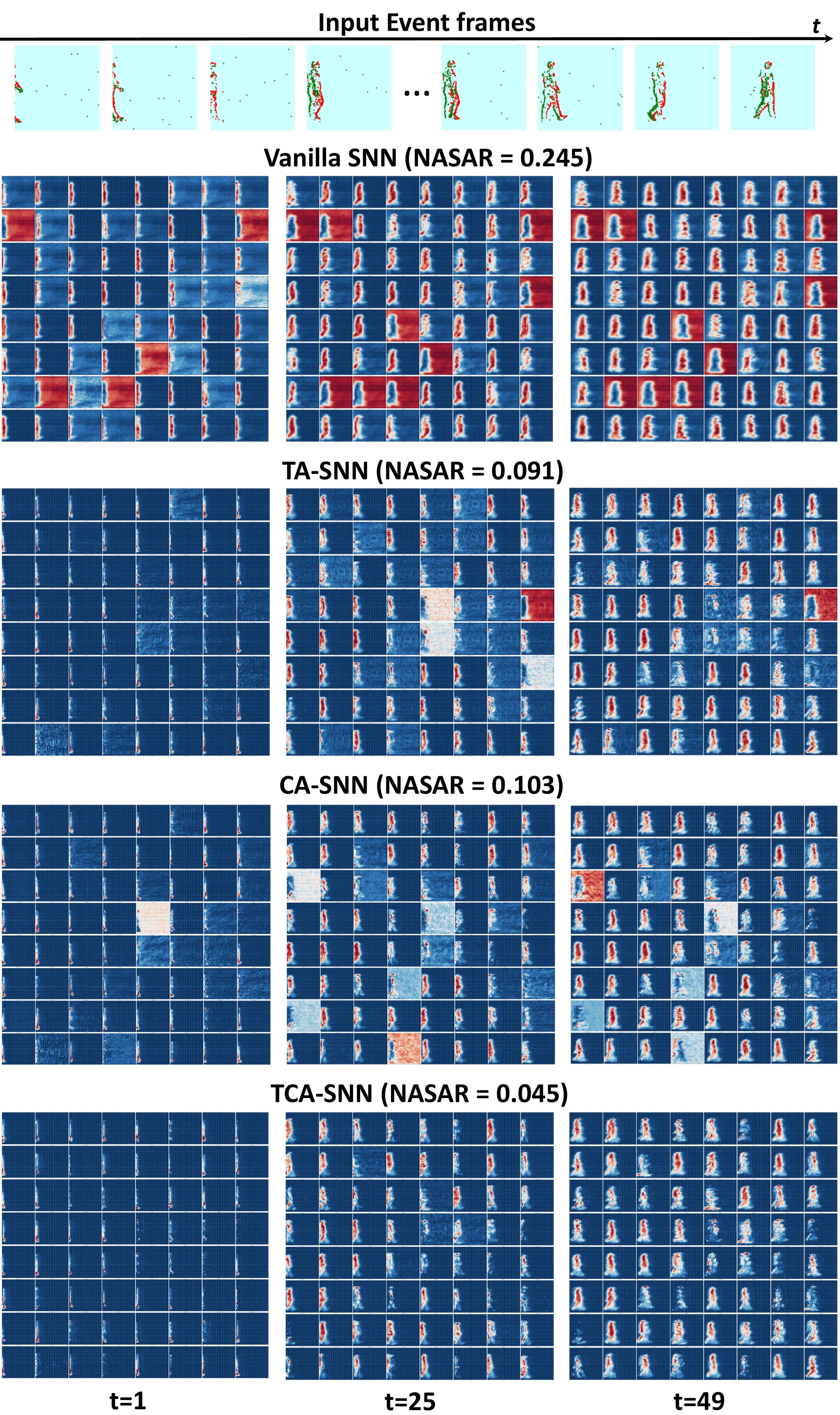}
\caption{Visualization of overall spiking response on Gait. From top to bottom: input event frames on Gait. Visualization of spiking response features in vanilla SNN, TA-SNN, CA-SNN, and TCA-SNN respectively, where $t=1,25,49$ and $n = 1$. Each pixel on the feature represents the spiking activity rate of one neuron over the whole validation set. For a single channel, the redder the pixel, the higher spiking activity rate; the bluer the pixel, the closer the spiking activity rate is to 0. The dimension of $\overline{\boldsymbol{S}}^{t, 1}$ (first layer) is $(32, 32, 64)$ at each time step, and we depict all 64 channels. We can clearly observe that attention drives the network to focus on the target and suppress the redundant background channels. Masking the background channel will significantly reduce the SNN energy cost, since the background contains a high spiking activity rate.} 
\label{Fig:Attention_Feature_Map}
\end{figure}

\subsection{Analysis of Overall Spiking Response}

Classical CNN visualization methods before-mentioned generally can only be exploited to analyze a single simple such as one image, providing an intuitive feel of the attention mechanism. Previous work \cite{kim_2021_SNN_CAM} extended CNN visualization methods to the SNN community, but did not get rid of the single-sample analysis limitation. On the other hand, existing attention works usually neglect the suppression part of the attention, which could dominate the network efficiency in SNNs. Although some works\cite{CBAM,park_2020_attention_bam,wang_2017_attention_background} have observed that attention would diminish background responses, they didn't realize the significance of this phenomenon. The underlying reason is that even if the background is suppressed to zero, it still needs to be computed and consumes energy for ANN on the GPU. By contrast, suppressing background noise is critical to the efficiency of SNNs because we observe that the background part has a higher spiking activity rate. 

In this paper, we propose the average spiking response visualization (ASRV) method to demonstrate the spiking response distribution of SNNs on various datasets. Specifically, we first compute the spiking tensor $\boldsymbol{S}^{t, n}$ (spiking feature maps, only 0 or 1) for each sample of the validation set. Then average all spiking tensors to get average spiking response feature $\overline{\boldsymbol{S}}^{t, n}\in\boldsymbol{R}^{c_{n} \times h_{n} \times w_{n}}$, where each element of $\overline{\boldsymbol{S}}^{t, n}$ represents the spiking activity rate of a neuron on the validation set. Finally, we plot $\overline{\boldsymbol{S}}^{t, n}$ to visualize the spiking response.

\textbf{Analysis of overall spiking response on DVS128 Gait.} In Fig.~\ref{Fig:Attention_Feature_Map}, we visualize average spiking response $\overline{\boldsymbol{S}}^{t, n}$ (take $n=1, t=1,25,49$ as examples) for various models based on Gait with $dt=15, T=60$, including vanilla SNN, TA-SNN, CA-SNN and TCA-SNN. Each pixel on the feature represents the spiking activity rate of one neuron over the whole validation set. We can clearly observe that attention modules drive the network to focus on the target and suppress the redundant background channels. For example, we see vanilla SNN has nine channels with a large area of red (spiking activity rates of these neurons are very close to 1) when $T=49$, which means that these channels focus on background information. After integrating with attention modules, the background channels are significantly suppressed. Thereby the NASAR of SNNs is greatly reduced. Specifically, the NASAR of TA-SNN and CA-SNN has been respectively dropped from 0.245 to 0.091 and 0.103. Combining TA and CA (i.e., TCA) can further drop NASAR to 0.045 without invalid background information. In short, we discover that attention can significantly suppress unimportant background channels that contain a very high spiking activity rate, which in turn drops the energy cost of SNNs. Crucially, the lower the NASAR, the higher the energy efficiency because the neuromorphic chip can skip the computation of zeros. 

\begin{figure}[!t]
\centering
\subfigure[Earlier stage of overall spiking response]{\includegraphics[scale=0.22]{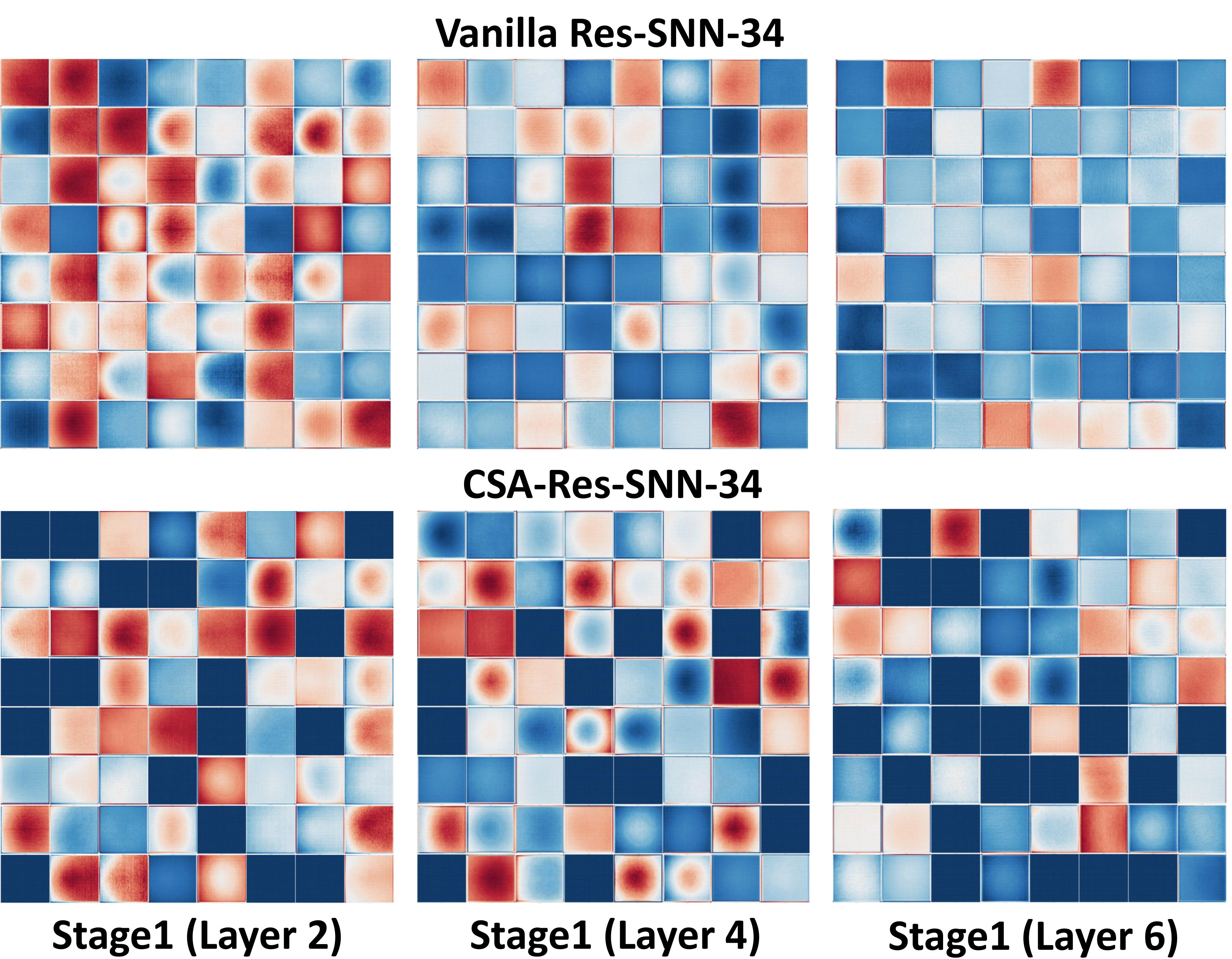}}
\subfigure[Last stage of overall spiking response]{\includegraphics[scale=0.2]{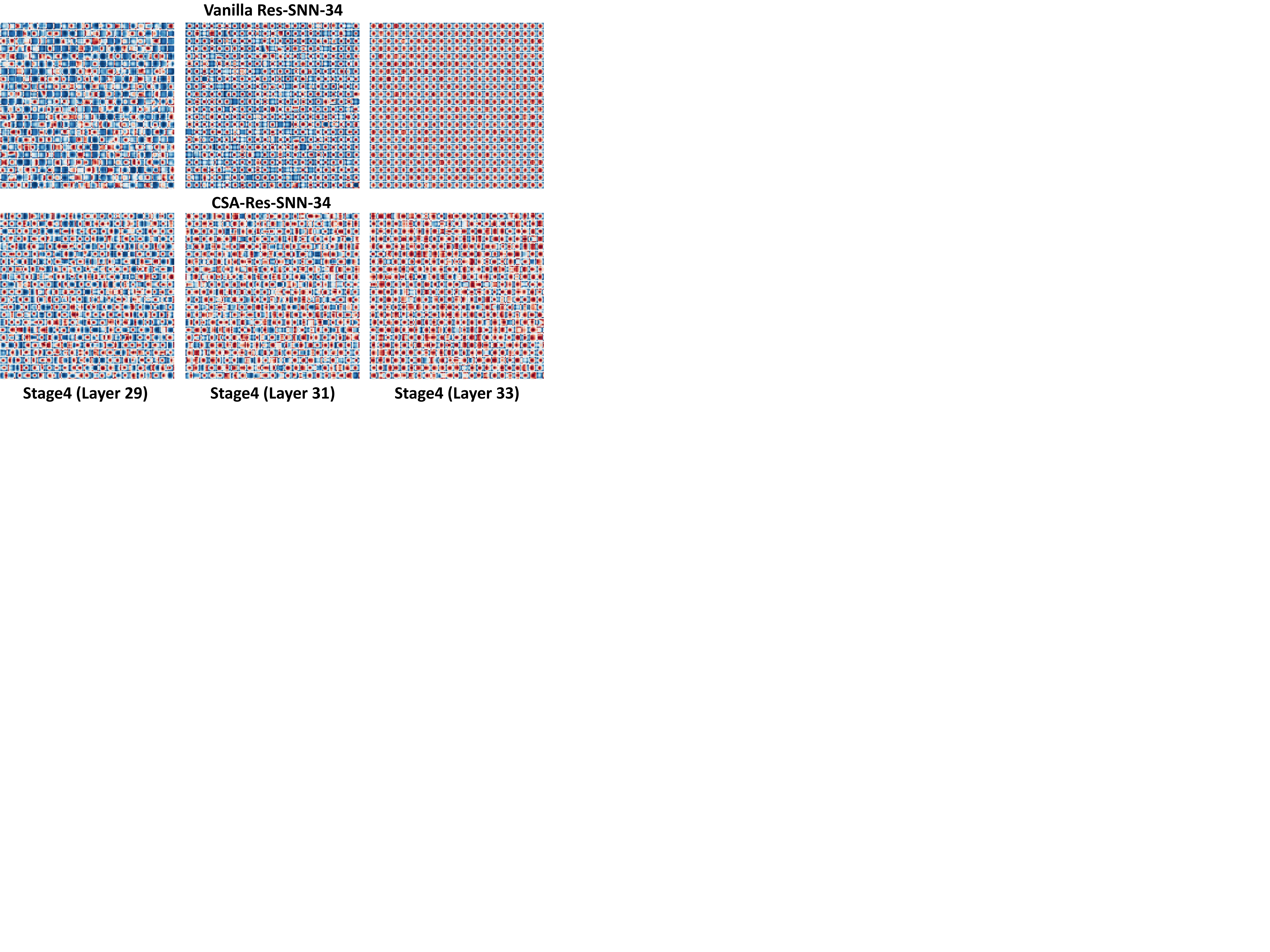}}
\caption{Visualization of overall spiking response on ImageNet-1K. (a) Earlier stage. In contrast to the vanilla model, some dark blue (all neurons in these channels are not spiking at all) channels appear in the stage 1 of CSA-Res-SNN, which induces great energy efficiency. For example, there are 17, 13, 21 dark blue channels in layer 2, 4, and 6 respectively. (b) Last stage. With the assistance of attention, features (red regions) of the last Conv layer (layer 33) become more diverse, which helps to enhance the class selectivity of the deep model.} 
\label{Fig:Attention_Feature_Map_ImageNet}
\end{figure}

\textbf{Analysis of overall spiking response on ImageNet-1K.} Each sample in the Gait is a person walking in front of a DVS camera, the only difference being that each person's gait was different. So in this single task dataset, we can easily observe the effect of attention on shallow plain SNNs. By contrast, when we compute the $\overline{\boldsymbol{S}}^{t, n}$ over large-scale datasets such as ImageNet-1K with 1,000 categories, it is difficult to observe specific objects in the channels. Nevertheless, we execute our ASRV method on ImageNet-1K and give the visualization results in Fig.~\ref{Fig:Attention_Feature_Map_ImageNet}, and still can observe the suppression of background channels caused by attention in deep SNNs. It is well known that at the early stage, the filters of deep networks tend to extract class-agnostic low-level features while extracting class-specific high-level features at the last stage. To investigate the effect of attention on low-level and high-level feature extraction of deep SNNs, we plot average spiking features at stage 1 (layers 2, 4, and 6 with 64 channels) and stage 4 (layers 29, 31, and 33 with 512 channels). 

We first observe stage 1 (earlier stage). The spiking response of the SNN is averaged from 50,000 samples. It is hard to see specific objects in each channel. Here we focus on the shift of spiking response caused by attention. In contrast to the baseline model, some dark blue channels appear in CSA-Res-SNN. For example, there are 17, 13, and 21 dark blue channels in layers 2, 4, and 6, respectively. We check all these dark blue channels carefully, finding that all neurons in these channels have a firing rate of 0. That is to say, these dark blue channels are suppressed to zero by attention. This phenomenon is interesting, which indicates attention can suppress some low-level features (probably useless background noise information) and induces sparser spiking activity at the early stage of deep SNNs. 

Then we observe stage 4 (last stage). At great depth, we see that the types of objects (red region in a channel) within each channel become diverse in Att-Res-SNN-34. Especially in the last Conv layer (layer 33), the red regions of objects in the channel of the attention model are significantly richer than the vanilla model counterpart. Previous works\cite{zhou_2016_CAM,SE_PAMI} demonstrate that features at the last Conv layer are critical to correct classification. Namely, earlier layer (e.g., stage 1) features are typically more general and re-used within the network. While the later layer (e.g., stage 4) features exhibit great levels of specificity. We conjecture that the attention module helps class selectivity by diversifying features of the last Conv layer, which brings a significant performance gain to Res-SNN-34 (+5.0 percent). 

\subsection{Spiking Response of Attention SNNs}
We plot the spiking response of vanilla SNN and TCA-SNN on Gait (Fig.~S1 in Section~S4 of SM), and we observe that the NASR of vanilla SNN is almost unchanged at each time step, which means SNN responds similarly to various inputs. This phenomenon is unreasonable because event streams are sparse and non-uniform\cite{yao_2021_TASNN}. With the help of data-dependent attention, the NSAR of TCA-SNN is uneven and small at the temporal axis, which induces a much lower NASAR than vanilla SNN. Similar experimental results could also be found in Gesture. Furthermore, we count the NASAR values of vanilla Res-SNN and Att-Res-SNN on ImageNet-1K with $T=1$. We find that the variance of NASAR is very small in vanilla Res-SNN, indicating the network's response to different images is almost invariant. In contrast, the variance of NSAR in Att-Res-SNN is more prominent. These observations demonstrate that attention can produce instance-specific dynamic responses.

Actually, making the spiking activity of SNNs sparser is always a fascinating topic because the human brain is a paragon model of sparse and efficiency\cite{Nature_2}. Current SNN models mainly drop NASAR by activity regularization\cite{deng_TNNLS_2021} or network compression\cite{lien_2022_sparse_compress}. But forcing sparsity too much may hurt predictive performance for an equal number of neurons since the effective capacity of the model might be reduced\cite{glorot_deep_sparse_2011}. Thus, in these \emph{parameter regularization} methods, the reduction of NASAR is limited, or the performance only holds a slight improvement. By contrast, our strategy is to optimize the membrane potential of spiking neurons in a \emph{data-dependent} way. We argue that reasonable membrane potential optimization based on specific features can naturally induce sparser spiking activity and better performance of SNNs concurrently.

\section{Conclusion}\label{section:conclusion}
In this work, we propose a lightweight attention module for SNNs, named multi-dimensional attention (MA), to boost both the performance and energy efficiency of SNNs. MA module is a plug-and-play that can be easily implemented and integrated with existing Conv-based SNNs. Inspired by attention theories in the neuroscience, our module optimizes the membrane potential of spiking neurons by learning when, what and where to focus and suppress through three separate pathways, which in turn drops the spiking activity and improves the performance. A wide range of experiments show the effectiveness and efficiency of MA, which achieve state-of-the-art performance and significant energy efficiency across multiple datasets and tasks, including event-based DVS128 Gesture/Gait and ImageNet-1K. We analyze how and why sparser spiking activity caused by attention is better, by visualizing the spiking response of vanilla and attention SNNs. We argue that effectiveness and efficiency exist as two sides of an elegant coin that should go hand in hand, which can be naturally symbiotic in SNNs, like coexistence in the human brain. The attention mechanisms can help us achieve this point.







\ifCLASSOPTIONcompsoc
  \section*{Acknowledgments}
This work was partially supported by Beijing Natural Science Foundation for Distinguished Young Scholars (JQ21015) and National Key R\&D Program of China (2018AAA0102600), and Beijing Academy of Artificial Intelligence (BAAI) and Pengcheng Lab.
  
\else
  \section*{Acknowledgment}
\fi

\bibliographystyle{IEEETran}
\bibliography{./ref}

\appendices

\ifCLASSOPTIONcaptionsoff
  \newpage
\fi

\end{document}